\documentclass{article}

\usepackage{booktabs}
\usepackage[table]{xcolor}
\usepackage{multirow}

\usepackage{microtype}
\usepackage{graphicx}
\usepackage{subcaption}
\usepackage{booktabs}

\usepackage{hyperref}

\usepackage[accepted]{icml2026}

\usepackage{amsmath}
\usepackage{amssymb}
\usepackage{mathtools}
\usepackage{amsthm}

\usepackage[capitalize,noabbrev]{cleveref}

\theoremstyle{plain}
\newtheorem{theorem}{Theorem}[section]
\newtheorem{proposition}[theorem]{Proposition}
\newtheorem{lemma}[theorem]{Lemma}

\theoremstyle{definition}
\newtheorem{definition}[theorem]{Definition}

\theoremstyle{remark}
\newtheorem{remark}[theorem]{Remark}

\usepackage[textsize=tiny]{todonotes}

\icmltitlerunning{From Outcomes to Actions: Leveraging Hindsight for Long-Horizon Language Agent Training}

\usepackage[most]{tcolorbox}

\tcbset{
  aibox/.style={
    width=\linewidth,
    top=8pt,
    bottom=4pt,
    left=2pt,
    right=2pt,
    colback=blue!6!white,
    colframe=black,
    colbacktitle=black,
    enhanced,
    center,
    attach boxed title to top left={yshift=-0.1in,xshift=0.15in},
    boxed title style={boxrule=0pt,colframe=white,},
  }
}
\newtcolorbox{AIbox}[2][]{aibox,title=#2,#1}
\definecolor{lightblue}{rgb}{0.22,0.45,0.70}

\begin{document}

\twocolumn[

  \icmltitle{From Outcomes to Actions: Leveraging \\
Hindsight for Long-Horizon Language Agent Training}

  \icmlsetsymbol{equal}{*}

  \begin{icmlauthorlist}
    \icmlauthor{Zishang Jiang}{sds}
    \icmlauthor{Tingyun Li}{sds}
    \icmlauthor{Jinyi Han}{saie}
    \icmlauthor{Xinyi Wang}{sds}
    \icmlauthor{Sihang Jiang}{ccsai}
    \icmlauthor{Yizhou Ying}{ccsai}
    \icmlauthor{Xiaojun Meng}{huawei}
    \icmlauthor{Jiansheng Wei}{huawei}
    \icmlauthor{Jiaqing Liang}{sds}
    \icmlauthor{Yanghua Xiao}{ccsai}
  \end{icmlauthorlist}

  \icmlaffiliation{sds}{School of Data Science, Fudan University}
  \icmlaffiliation{saie}{Shanghai Institute of Artificial Intelligence for Education, East China Normal University}
  \icmlaffiliation{ccsai}{College of Computer Science and Artificial Intelligence, Fudan University}
  \icmlaffiliation{huawei}{Huawei Large Model Data Technology Lab}

  \icmlcorrespondingauthor{Jiaqing Liang}{liangjiaqing@fudan.edu.cn}

  \icmlkeywords{Machine Learning, ICML}

  \vskip 0.3in
]

\printAffiliationsAndNotice{}

\begin{abstract}
Reinforcement learning (RL) has become a widely adopted technique for improving large language models (LLMs) on complex tasks. Despite this progress, existing RL methods still face challenges in training agents with longer-horizon interactions. One major bottleneck is distinguishing the contribution of different actions in long-horizon interaction, leading to high optimization variance. To address this, we introduce a novel policy gradient method, \textbf{H}indsight \textbf{P}olicy \textbf{O}ptimization (\textbf{HPO}), that projects both the current policy distribution and the hindsight distribution into an intent space and extracts low-variance learning signals from the Wasserstein distance between them. We theoretically and empirically show that aggregating semantically similar states and actions in the intent space yields a bounded-variance estimator and improves policy performance stably. Our code is available online\footnote{https://github.com/Jiangzs1028/hpo}.

\end{abstract}

\section{Introduction}
Reinforcement learning (RL) has become a widely adopted technique for improving large language models (LLMs) on complex decision-making tasks. Compared to earlier prompt-based approaches, RL allows LLM-based agents to actively explore and learn from interaction with the environment, and has shown strong potential in applications such as tool use~\cite{jin2025search,cao2025skyrl}, smart device operation~\cite{luo2025gui}, and interactive environments~\cite{abdulhailmrl,hu2024survey}.

\begin{figure}[t]
    \centering

    \includegraphics[width=0.9\linewidth]{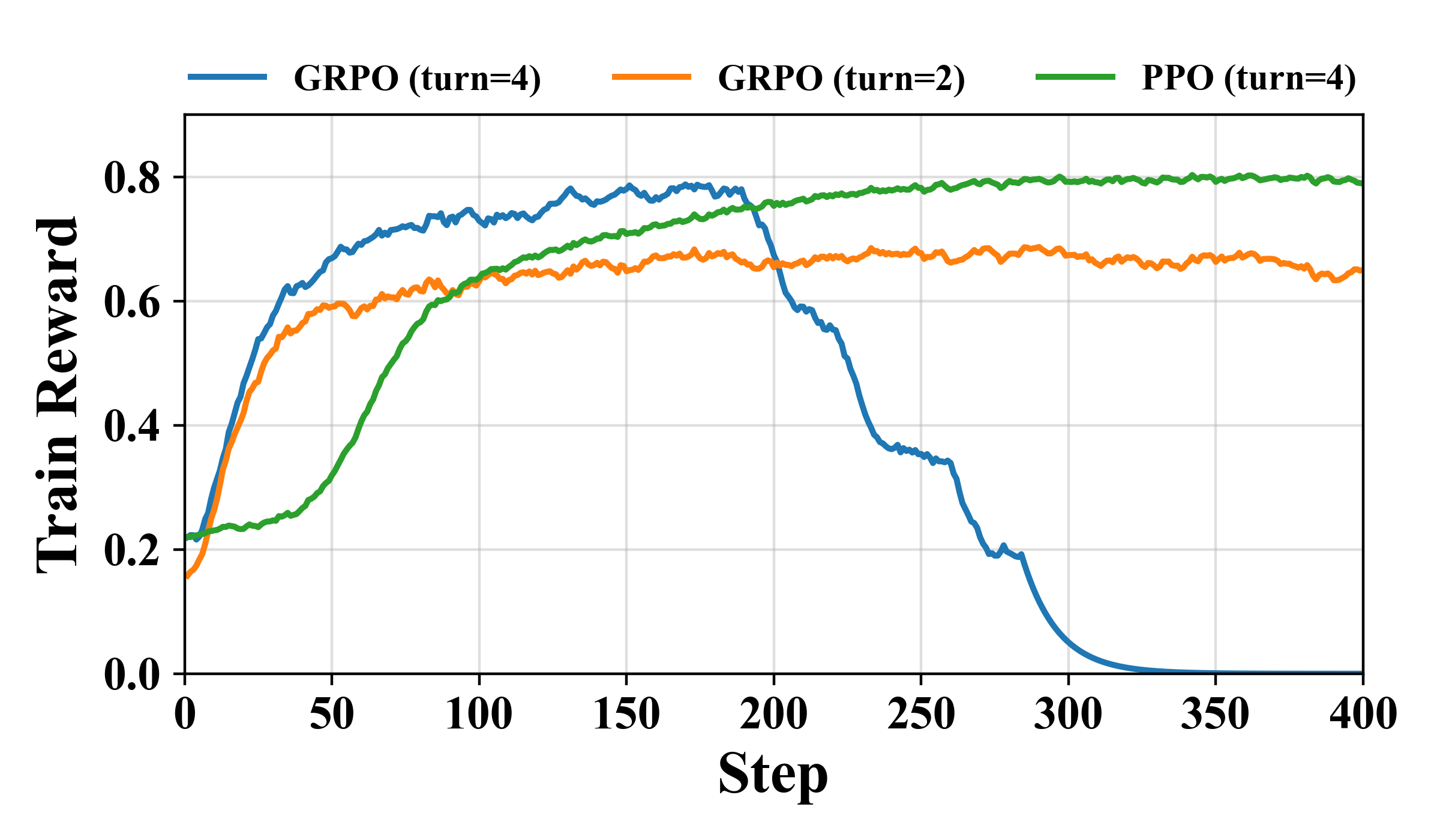}
    \caption{Existing RL methods struggle in long-horizon interactions. GRPO becomes unstable with longer-horizon interactions, while PPO remains stable but converges slowly due to critic warmup.}
    \label{fig:intro_plot}
    \vspace{-1.4em}
\end{figure}

Despite this progress, existing RL algorithms still face challenges in training agents for longer-horizon interactions. As illustrated in Figure \ref{fig:intro_plot}, GRPO~\cite{shao2024deepseekmath} can maintain stable training under shorter interaction horizons, but becomes unstable as the horizon length increases, often exhibiting sudden reward collapse~\cite{jin2025search,xue2025simpletir,sun2025zerosearch}. In contrast, PPO~\cite{schulman2017proximal} achieves more stable training behavior, but typically converges more slowly due to the additional requirement of training a critic model.

One major bottleneck is distinguishing the contribution of different actions in long-horizon interaction, leading to high optimization variance. Unlike single-step tasks, multi-turn decision-making requires attributing outcomes to a sequence of interdependent actions, where effective and ineffective actions are often interleaved within the same trajectory~\cite{zhang2025landscape}. However, RL methods that rely solely on outcome rewards, such as GRPO, typically propagate the final outcome uniformly across all actions, which may mistakenly reinforce some ineffective actions and thereby destabilize training~\cite{zhang2025landscape}. In contrast, PPO uses a critic to better distinguish actions, leading to more stable training, but it converges more slowly due to the warm-up required by the critic~\cite{jin2025search}.

We define the hindsight distribution as the optimal action distribution induced by an agent retrospectively reselecting its early actions after observing the final outcome, and show that existing RL methods with Monte Carlo estimation can be interpreted as a \textit{point-wise} comparison of the \textbf{discrepancy between the current policy distribution and this hindsight distribution} \textit{in a discrete state–action space}. This point-wise comparison \textbf{ignores the underlying semantic structure of the language action space}, resulting in unbounded variance in exponentially large language action spaces.

\textbf{Although the action space of large language models is vast, the set of underlying intents they can express is comparatively limited.} For example, \textit{Who is the president of the United States?} and \textit{I want to know who the U.S. president is} correspond to distinct language actions but convey nearly identical semantic intent.

Based on this insight, we introduce a novel policy gradient method, \textbf{H}indsight \textbf{P}olicy \textbf{O}ptimization (\textbf{HPO}), that projects both the current policy distribution and the hindsight distribution into an intent space and measures their discrepancy using the Wasserstein distance~\cite{villani2008optimal}. We theoretically show that aggregating semantically similar states and actions in this lower-dimensional intent space improves statistical efficiency and yields a bounded-variance estimator.

Our contributions are as follows:
\begin{enumerate}
\item We reformulate existing Monte Carlo policy gradient methods from a hindsight distribution perspective, and show that point-wise comparisons in the discrete language action space can lead to unbounded variance due to ignoring the underlying semantic structure.
\item We propose HPO, a principled approach that constructs hindsight distributions in an intent space and derives bounded-variance learning signals for smoother and more sample-efficient policy optimization, with negligible computational overhead.
\item We evaluate HPO on multiple long-horizon tasks and show that it substantially reduces training noise and stabilizes optimization. Further experiments show that such learning signals are interpretable and remain sufficient to drive effective policy learning even without outcome rewards.
\end{enumerate}

\section{Preliminaries}
\subsection{Task Formulation}

We formulate language action tasks as a Markov decision process (MDP)
$\mathcal{M}=(\mathcal{S},\mathcal{A},\mathcal{T},\mathcal{R}, \gamma)$,
where $\mathcal{S}$ denotes the state space,
$\mathcal{A}$ is the action space of natural language actions,
$\mathcal{T}:\mathcal{S}\times\mathcal{A}\rightarrow\mathcal{S}$ is the transition function,
$\mathcal{R}$ is the reward function, and
$\gamma\in(0,1]$ is the discount factor.
In the RLVR setting, rewards are sparse and only provided at the end of each episode and the discount factor is typically set to $\gamma=1$. When the context is clear, we omit $\gamma$ for notational simplicity.
An agent interacts with the environment over discrete time steps.
At each step $t$, the agent observes the current state $s_t\in\mathcal{S}$
and selects an action $a_t\sim\pi_\theta(\cdot\mid s_t)$.
The environment then transitions to the next state $s_{t+1}$
according to $\mathcal{T}(s_{t+1}\mid s_t,a_t)$.
The objective is to maximize the expected cumulative reward under the policy $\pi_\theta$.

\subsection{Value Function} For $M_\pi=(S_0,A_0, S_1,A_1,\cdots)$, there is an occupancy measure $\rho_\pi$ which satisfies $\rho_\pi(s,a)=(1-\gamma)\sum_{t=0}^\infty\gamma^t\mathbb{P}(S_t =s, A_t=a)$. Moreover, there is a one-to-one correspondence between those measures and the policies.
The state-value function under policy $\pi$ is defined as $V_\pi(s)\triangleq \mathbb{E}_{\pi}\!\left[\sum_{t=0}^{\infty}\gamma^t r_t \mid S_0=s\right]$, which measures the expected discounted return starting from state $s$. Similarly, the action-value function is defined as $Q_\pi(s,a)\triangleq \mathbb{E}_{\pi}\!\left[\sum_{t=0}^{\infty}\gamma^t r_t \mid S_0=s, A_0=a\right]$.

\subsection{Wasserstein Distance}
Wasserstein distance, originating from optimal transport theory~\cite{villani2008optimal}, measures the discrepancy between two probability distributions by the minimum cost of transporting mass from one distribution to the other. Given two probability measures $\mu_A$ and $\mu_B$ on a metric space $(\mathcal{X}, \text{dis})$, the 1-Wasserstein distance is defined as
\begin{equation}
W_1(\mu_A, \mu_B)
=
\inf_{\pi \in \Pi(\mu_A, \mu_B)}
\int_{\mathcal{X}\times\mathcal{X}} \text{dis}(x_A, x_B)\, d\pi(x_A, x_B),
\end{equation}
where $\Pi(\mu_A, \mu_B)$ is the set of couplings whose marginals are $\mu_A$ and $\mu_B$. Intuitively, each coupling can be viewed as a possible transportation plan that describes how mass from $\mu_A$ is assigned to mass in $\mu_B$. In practice, $W_1$ is often computed via its dual formulation:

\vspace{-1.5em}
\begin{small}
\begin{equation}
W_1(\mu_A,\mu_B)
=
\sup_{\|f\|_{L}\le 1}
\left(
\mathbb{E}_{x\sim\mu_A}[f(x)]
-
\mathbb{E}_{x\sim\mu_B}[f(x)]
\right),
\end{equation}
\end{small}
\vspace{-1.5em}

where the maximization is over all 1-Lipschitz functions $f$. The optimal function $f^\star$ in this dual problem is referred to as the \emph{Kantorovich potential}~\cite{villani2008optimal}.
\label{k_potential}

\section{The Challenge of Long-Horizon Language Agent Reinforcement Learning}
Existing studies~\cite{jin2025search,xue2025simpletir,sun2025zerosearch} have found that RL methods become increasingly unstable as the interaction horizon grows. In this section, we show that this instability arises from the high-variance gradient noise induced by long-horizon interactions, which free-critic methods struggle to mitigate and may even further exacerbate.

\subsection{Rethinking the Objective in RL for LLMs}
For a unified analysis, common policy optimization methods used in RL for LLMs, including GRPO, REINFORCE++~\cite{hu2025reinforce++}, and PPO, can all be expressed in the following generic policy gradient form like REINFORCE~\cite{williams1992simple}:

\vspace{-0.8em}
{\small
\begin{equation}
    \nabla_\theta \mathcal{J}(\theta) = \mathbb{E}_{\tau \sim \pi_\theta}\left[\sum_{t=0}^{\infty}\nabla_\theta \log \pi_\theta(a_t \mid s_t)(\hat Q(s_t,a_t)-b_t)\right]
\end{equation}
\small}

\vspace{-0.8em}
where $\hat Q(s_t, a_t)$ denotes an estimator of $Q_\pi(s_t,a_t)$.

For example, the objective in GRPO can be interpreted as repeatedly sampling $G$ trajectories from a shared initial state $s_0$. For analytical convenience, we assume equal return variance across different \(s_0\). Concretely, GRPO optimizes

\vspace{-0.8em}
{\small
\begin{align}
&\nabla_\theta \mathcal{J}_{\text{GRPO}}(\theta)
=
\mathbb{E}_{\{\tau_i\}_{i=1}^{G} \sim \pi_\theta} \nonumber \\
&\left[
\frac{1}{G}
\sum_{i=1}^{G}
\sum_{t=0}^{\infty}
\nabla_\theta \log \pi_\theta(a_{i,t} \mid s_{i,t})
\big(\hat Q(s_{i,t}, a_{i,t}) - b_{i,t}\big)
\right],
\end{align}
\small}

\vspace{-0.8em}
where \(\hat Q(s_{i,t}, a_{i,t})=R(\tau_i)\) and $b_{i,t}=\frac{1}{G}\sum_{j=1}^{G}R(\tau_j)$. The latter can be viewed as a Monte Carlo (MC) estimate \(\bar V_\pi(s_0)\) of \( V_\pi(s_0)\).

Although MC estimation provides an unbiased estimate of $Q_\pi$, it suffers from excessive variance by accumulating stochasticity from future actions~\cite{sutton1998reinforcement}. For example, even if the agent takes a correct action at the current step, the final return still depends on every later sampled action: a wrong one can turn an otherwise successful trajectory into a failure. Therefore, the same current action may receive different terminal returns across rollouts, and this variance can grow exponentially with the horizon.

It is well known that excessive variance can hinder effective training updates, so existing RL algorithms typically seek to reduce training variance. Algorithms such as GRPO introduce a fixed baseline to reduce variance without introducing bias, and have achieved strong performance on single-turn tasks such as mathematical reasoning. However, we show below that the effectiveness of this approach diminishes significantly in long-horizon interactions.

\begin{lemma}[Variance Decomposition]
\label{thm:var_decomp_baseline}
Consider the REINFORCE gradient estimator $\hat{G}_b$ with the baseline $b$,
\begin{equation}
\hat{G}_b \;=\; g(a)\,\big(\hat Q(s,a)-b(s)\big),
\end{equation}
where $a \sim \pi(\cdot\mid s)$ and $g(a)=\nabla_\theta \log \pi_\theta(a\mid s)$.

Let $\|\cdot\|$ denotes the Euclidean norm. Then its conditional variance admits the decomposition
\vspace{-0em}
\begin{small}
\begin{align}
\mathrm{Var}(\hat{G}_b \mid s)\!=\!\mathbb{E}\!\left[\|g(a)\|^2\mid s\right]
\!\cdot\!
\mathrm{Var}\!\left(\hat Q(s,a)-b(s)\mid s\right)
\label{var_decomp}
\end{align}
\end{small}
\end{lemma}

\begin{remark}
Eq.~\eqref{var_decomp} shows that the variance of the policy gradient can be decomposed into the product of \textbf{the gradient norm} $\mathbb{E}\!\left[\|g(a)\|^2\mid s\right]$ and \textbf{the variance of the baseline-adjusted $Q$ signal} $\hat Q(s,a)-b(s)$.
\end{remark}

\paragraph{The gradient norm.}
The gradient-norm factor in Eq.~\eqref{var_decomp} can magnify optimization noise. With rounds of external tool feedback, some low-probability actions are increasingly sampled (e.g., meaningless or repetitive tokens)~\cite{xue2025simpletir}. Due to the logarithmic form of the policy gradient, these rare actions can induce disproportionately large, and even unbounded gradient norms, which further amplify the noise variance caused by state mismatch in Eq.~\eqref{var_decomp}. This perspective also helps explain why filtering strategies that remove low-probability samples can mitigate collapse~\cite{xue2025simpletir}.

\paragraph{The variance of the baseline-adjusted $Q$ signal.}
The second factor of the gradient variance is the variance of the baseline-adjusted $Q$ signal, $\hat Q(s,a)-b(s)$. In principle, there exists an optimal baseline that minimizes this variance, but it is generally difficult to compute in practice. Existing methods therefore approximate it in different ways: GRPO uses a fixed baseline, while PPO learns a critic-based baseline.

\begin{theorem}[Optimal Baseline and Excess Variance]
\label{thm:optimal_baseline}
Among all scalar baselines $b(s)$ independent of the sampled action $a$, the conditional variance
$\mathrm{Var}(\hat G_b \mid s)$ is minimized by
\begin{equation}
b^*(s) =
\frac{\mathbb{E}\!\left[\|g(a)\|^2 \hat Q(s,a)\mid s\right]}
{\mathbb{E}\!\left[\|g(a)\|^2\mid s\right]}.
\end{equation}
Moreover, for any alternative baseline $\tilde b(s)$, the increase in variance admits the exact expression
\begin{align}
    \mathrm{Var}(\hat{G}_{\tilde b} \mid s)& \,-\,  \mathrm{Var}(\hat{G}_{b^*} \mid s), \,=\,   \nonumber\\
&\mathbb{E}\!\left[\,\|g(a)\|^2\mid s\right](\tilde b(s)-b^*(s))^2.
\end{align}
\end{theorem}
\vspace{-0.8em}

We remark that even though the optimal state-only baseline is known, it is rarely used in practice due to its computational intractability~\cite{li2024remax,duan2016benchmarking}. Rather, for both computational and conceptual benefit, the choice of $\mathbb{E}[\hat Q(s,a)\mid s]$ is often used~\cite{schulman2017proximal,wu2018variance}. In particular, when $\|g(a)\|^2$ and $\hat Q(s,a)$ are loosely correlated, this baseline is close to the optimal baseline. For ease of analysis and empirical evaluation, we follow this setting.

Under this proxy, the minimum-variance baseline for MC returns is the state-dependent value function $V_\pi(s_t)=\mathbb{E}[R\mid s_t]$. GRPO approximate it with a fixed baseline $\bar V_\pi(s_0)$ shared across all subsequent time steps. This approximation can be effective in single-turn tasks, but becomes inaccurate in long-horizon, tool-augmented interactions, where $V_\pi(s_t)$ can vary dramatically across states: for instance, after a successful search, the agent is much more likely to find the correct final answer, whereas after a failed search, its chance of success may drop sharply.

To further analyze how this mismatch changes with the interaction horizon, we conduct a preliminary multi-turn Search experiment, where states are grouped by identical environment observations to estimate $V_\pi(s_t)$. As shown in Figure~\ref{fig:state_diff}, as the interaction horizon increases, the mismatch $(V_\pi(s_t)-\bar V_\pi(s_0))^2$ can grow large, directly weakening the variance-reduction effect of the baseline by Theorem~\ref{thm:optimal_baseline} and, in extreme cases, leading to higher gradient variance than using no baseline at all.

\begin{figure}[h]
    \centering

    \includegraphics[width=\linewidth]{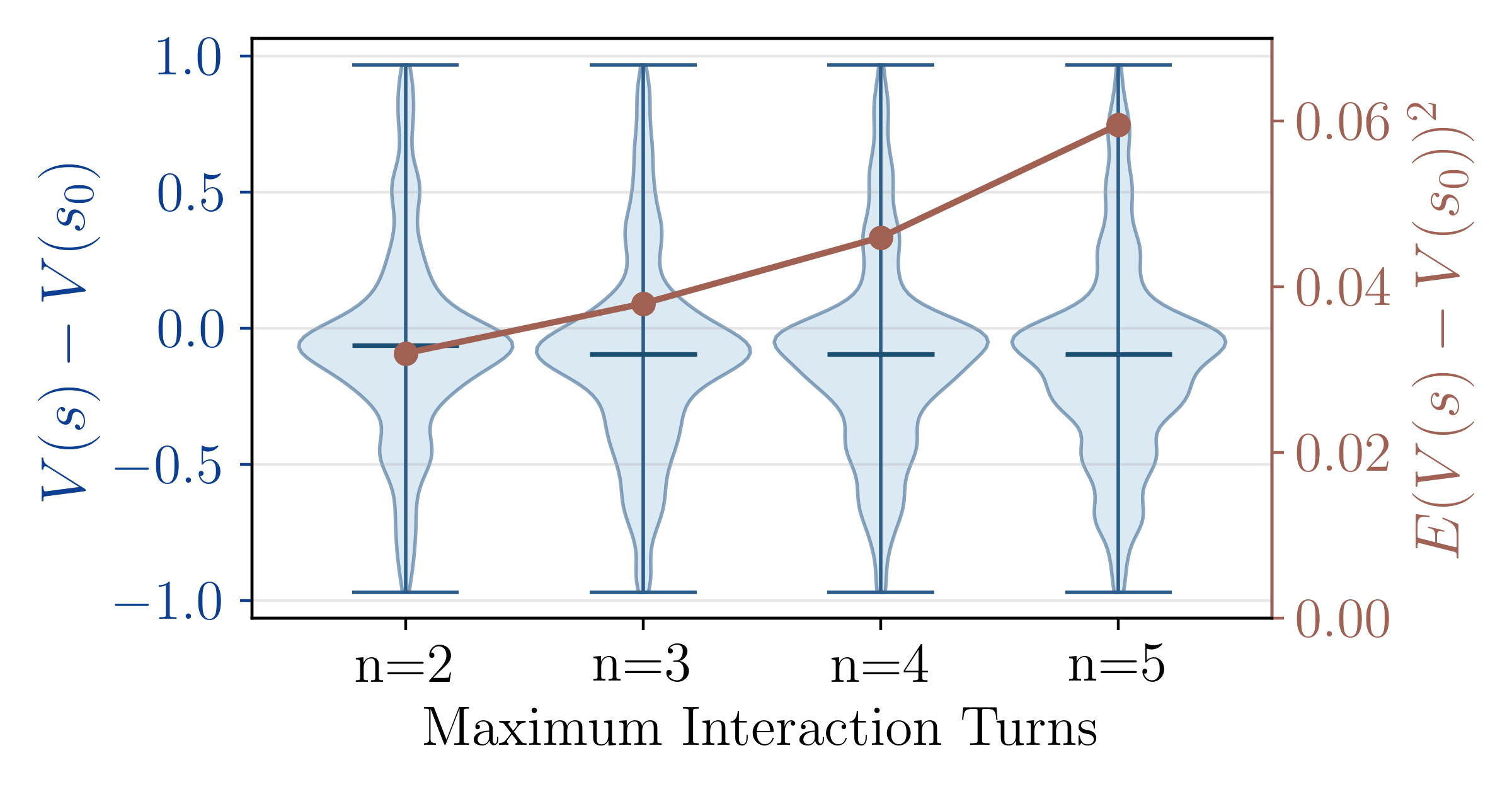}
    \caption{Multi-turn Search experiment where states are grouped by identical environment observations. As the maximum interaction turns $n$ grows, the mismatch between $V_\pi(s_t)$ and the fixed baseline $\bar V_\pi(s_0)$ can become increasingly severe.}
    \label{fig:state_diff}

\end{figure}

\vspace{-0.1em}
\begin{AIbox}{Summary}

As the interaction horizon increases, the variance of RL training can grow substantially, while existing critic-free methods struggle to reduce it and may even exacerbate it. This motivates the need for a estimator that remains low-variance in long-horizon interactions.
\end{AIbox}

\begin{figure*}[!t]
  \centering
  \includegraphics[width=0.84\textwidth]{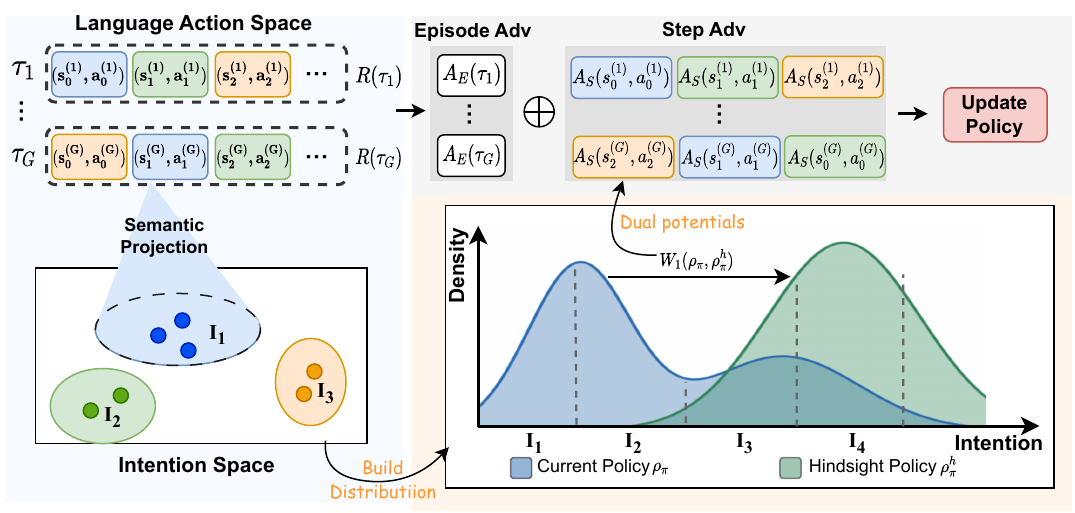}
  \caption{Illustration of HPO. HPO performs semantic projection and compares the policy distribution with the hindsight distribution in the intent space. By solving the Wasserstein distance between the two distributions, it obtains step-level advantages, which are combined with episode-level advantages to update the policy. }
  \label{fig:method_figure}
\end{figure*}

\section{Training LLM Agents with HPO}
\label{sec:hpo_method}
In this section, we introduce a novel policy gradient estimator. Unlike Monte Carlo estimators defined over discrete action spaces, it constructs learning signals in a semantically aggregated intent space. By sharing statistical information at the semantic level, it improves statistical efficiency and substantially reduces estimation variance.

\subsection{Policy Gradient from the Hindsight Perspective}
It is well known that RL becomes challenging in large action spaces. Compared with single-turn tasks, a $T$-turn interaction induces an action-sequence space whose size scales as $|\mathcal{A}|^T$, growing exponentially with the interaction horizon. For language agents, this difficulty is further amplified because each action is a natural-language output, which can itself be viewed as a joint choice over a sequence of tokens and thus induces an enormous action space.

To analyze the trajectory distribution, we consider their induced state--action occupancy measures over pairs $(s,a)$, and define the hindsight distribution and show that the policy gradient can be equivalently reformulated from the perspective of the hindsight trajectory distribution.

\begin{definition}[Hindsight Policy Distribution]
\label{def:hindsight_policy}
For any state $s$, action $a$, and return value $z$, let $h_z^\pi(a \mid s)$ denote the conditional probability that the first action equals $a$ when sampling a trajectory from $\pi$ starting at $s$, given that the realized return equals $z$. In the binary-reward setting considered in this work, we take $z=1$ and define the corresponding hindsight occupancy measure as
\begin{equation}
\label{eq:hindsight_ds}
\rho_\pi^h(s,a)
=
\sum_{t=0}^{\infty}
\mathbb{P}_\pi
\!\left(
(s_t, a_t) = (s,a)
\;\middle|\;
\sum_{t'=t}^{\infty} r_{t'} = 1
\right).
\end{equation}
\end{definition}

Intuitively, in the binary-reward setting, this distribution is obtained by retaining only state-action pairs from trajectories that eventually succeed. For continuous-reward settings, we provide the corresponding distribution definition in Appendix~\ref{app:general_hindsight}.

\begin{lemma}
\label{lem:grad}
Consider the objective of policy optimization $\mathcal{J}(\theta)=\mathbb{E}_{\tau \sim \pi}[\sum_{t\ge 0}r_t]$. Its policy gradient admits the following two equivalent forms:
\begin{align}
    \nabla_\theta &\mathcal{J}(\theta) = \mathbb{E}_{(s,a)\sim\rho_{\pi}}[Q_\pi(s,a) \nabla_\theta\log \pi_\theta (a\mid s) ] \\
    & =\mathbb{E}_{(s,a)\sim\rho_\pi}[-\frac{\delta KL(\rho^h_\pi || \rho_\pi)}{\delta \rho_\pi(s,a)}\nabla_\theta \log \pi_\theta (a\mid s)]
    \label{eq:kl_policy}
\end{align}
where $-\frac{\delta KL(\rho^h_\pi || \rho_\pi)}{\delta \rho_\pi(s,a)}=\frac{\rho^h_\pi(s,a)}{\rho_\pi(s,a)}$ denotes the point-wise variational derivative of $KL(\rho^h_\pi || \rho_\pi)$ with respect to $\rho_\pi(s,a)$.
\end{lemma}

A detailed proof is provided in Appendix~\ref{app:proof_lem42}. Lemma~\ref{lem:grad} provides an alternative interpretation of policy gradients from the perspective of the hindsight distribution: Monte Carlo policy-gradient methods can be viewed as comparing the current occupancy $\rho_\pi$ with the hindsight occupancy $\rho_\pi^h$ through the point-wise ratio $\rho_\pi^h(s,a)/\rho_\pi(s,a)$.

This point-wise comparison treats different state--action pairs as unrelated atoms. Such treatment can be effective in traditional RL problems with small action spaces, where the same state--action pair can be revisited frequently. However, in long-horizon language-agent training, the number of possible atoms grows with both the per-step language action space and the interaction horizon. Semantically similar actions with different surface forms are counted as distinct atoms and cannot share statistical evidence, making value estimation for language-agent actions highly inefficient.

This suggests that reducing the effective estimation space is necessary for obtaining stable learning signals. Although the action space of LLMs is exponentially large, the space of intents they can express is significantly lower-dimensional. Many distinct actions are semantically similar and lead to comparable outcomes; for example, \textit{Who is the president of the United States?} and \textit{I want to know who the U.S. president is} correspond to different language actions but express nearly identical underlying intent.

Based on this insight, the following sections introduce how to leverage the semantics of language actions to reduce the effective size of the action space and obtain more stable value estimation.

\subsection{Intent Space Construction}
\label{Sec:intention_space}
Given a state space $\mathcal{S}$ and an action space $\mathcal{A}$, each state–action pair $(s,a)$ is embedded into a $d$-dimensional Hilbert semantic space via a pretrained encoder $\Phi:\mathcal{S}\times\mathcal{A}\rightarrow\mathbb{R}^d$, where the Euclidean distance serves as the metric of the embedding space. In this semantic space, interactions that are lexically different but semantically equivalent are mapped to nearby regions.

In the binary-reward setting, the policy distribution $\rho_\pi$ and hindsight distribution $\rho_\pi^h$ are constructed as empirical distributions over embedded state–action pairs in the intent space.
\begin{align}
\label{eq:hindsight_intent}
\rho_\pi
= &
\sum_{t=0}^{\infty}
\mathbb{E}_{ \pi}
\!\left[
\delta_{\Phi(s_t,a_t)}
\right], \nonumber \\
\rho_\pi^h
= &
\sum_{t=0}^{\infty}
\mathbb{E}_{ \pi}
\!\left[
\mathbf{1}\!\left(\sum_{t'=t}^{\infty} r_{t'} = 1\right)
\;\delta_{\Phi(s_t,a_t)}
\right],
\end{align}
where \(\delta_{\Phi(s_t,a_t)}\) denotes the Dirac measure centered at the intent embedding \(\Phi(s_t,a_t)\).

\paragraph{Data for constructing the hindsight distribution.}
Although constructing the hindsight distribution from on-policy rollouts yields a closer approximation to the on-policy value function \(Q_\pi\), it suffers from unstable estimation due to fluctuating successful rollouts. Alternatively, using an offline set of successful trajectories obtained via rejection sampling from the initial policy provides a more stable construction of the hindsight distribution, particularly in sparse-reward scenarios, where on-policy rollouts rarely yield successful trajectories. We refer to these two variants as \textbf{online HPO} and \textbf{offline HPO}, respectively.

\subsection{Hindsight Policy Optimization}
Compared with existing RL estimators that use KL-induced point-wise estimates in the discrete state--action space, we seek a new value estimator that preserves the optimal hindsight policy while reducing value-estimation fluctuations among semantically similar language actions. In the intent space, this desideratum can be formalized as the following constrained scoring problem.

\begin{proposition}[Intent value estimatior]
\label{prop:wasserstein_dual_value}
Let $\mathcal{Z}=\Phi(\mathcal{S}\times\mathcal{A})$ be the intent space equipped with metric $d$, and let $\rho_\pi$ and $\rho_\pi^h$ denote the current and hindsight occupancy measures on $\mathcal{Z}$. A scoring function that identifies hindsight-optimal trajectories among trajectories under the semantic smoothness constraint $\|v\|_L\le 1$ is given by
\begin{equation}
\sup_{\|v\|_L\le 1}
\left(
\mathbb{E}_{z\sim\rho_\pi^h}[v(z)]
-
\mathbb{E}_{z\sim\rho_\pi}[v(z)]
\right).
\end{equation}
Any optimizer $v^*$ is a Kantorovich potential of the 1-Wasserstein distance.
\end{proposition}

We find that the intent value evaluator is exactly the dual problem of the Wasserstein distance between $\rho_\pi^h$ and $\rho_\pi$ introduced in Section~\ref{k_potential}. Under this view, the resulting policy gradient has the same form as Eq.~\eqref{eq:kl_policy}, except that the KL divergence is replaced by the Wasserstein distance.

Let $f^*$ denote the Kantorovich potential for $W_1(\rho_\pi,\rho_\pi^h)$. The corresponding variational gradient gives the following policy-gradient form:
\begin{equation}
    \nabla_\theta \mathcal{J}(\theta) =\mathbb{E}_{(s,a)\sim\rho_\pi}[-\frac{\delta W_1(\rho_\pi || \rho^h_\pi)}{\delta \rho_\pi(s,a)}\nabla_\theta \log \pi_\theta (a\mid s)]
\end{equation}
where $\frac{\delta W_1(\rho_\pi || \rho^h_\pi)}{\delta \rho_\pi(s,a)}=f^*(s,a)$ denotes the point-wise variational derivative of $W_1(\rho_\pi || \rho^h_\pi)$ with respect to $\rho_\pi(s,a)$, i.e., the Kantorovich potential in Section~\ref{k_potential}.

$f^*(s,a)$ possesses several desirable properties. Notably, it satisfies the 1-Lipschitz continuity condition in the intent space, implying that behaviors mapped to similar intents yield close potential values. Furthermore, it is also computationally efficient~\cite{ambrosio2005gradient,alvarez2021dataset} in that $f^*(s,a)$ can be recovered directly from the dual solution of the Wasserstein distance calculation at virtually no additional cost~\cite{alvarez2021dataset,just2023lava}.

As proved in Section~\ref{sec:the_ana},
$-\frac{\delta W_1(\rho_\pi \,\|\, \rho^h_\pi)}{\delta \rho_\pi(s,a)} = -f^*(s,a)$,
which is defined in the intent space, yields a learning signal with lower variance, at the cost of introducing a certain degree of bias.
In contrast, the learning signal $-\frac{\delta \mathrm{KL}(\rho^h_\pi \,\|\, \rho_\pi)}{\delta \rho_\pi(s,a)}=\frac{\rho^h_\pi(s,a)}{\rho_\pi(s,a)}$, which is defined over the discrete action space, is unbiased but typically exhibits much higher variance.

Since GRPO is currently one of the most widely adopted algorithms in RLVR, we choose it as the base optimization algorithm and apply the proposed hindsight distribution to its policy update. Concretely, the optimization terms induced by the KL divergence and the Wasserstein distance admit intuitive interpretations as \emph{episode-level} and \emph{step-level} advantages in GRPO, respectively.

\paragraph{KL term (Episode-level advantage).}
The KL-based optimization term corresponds to a Monte Carlo estimate of $Q_\pi(s,a)$.
Therefore, we directly reuse the episode-level advantage originally defined in GRPO as the learning signal for the KL term:
\begin{equation}
A_E(\tau_i)
=
\frac{R_i - \operatorname{mean}\!\left(\{R_i\}_{i=1}^{G}\right)}
{\operatorname{std}\!\left(\{R_i\}_{i=1}^{G}\right)} ,
\end{equation}

\paragraph{Wasserstein term (Step-level advantage).}
For the Wasserstein-based optimization term, we obtain the Kantorovich potential $-\frac{\delta W_1(\rho_\pi \,\|\, \rho^h_\pi)}{\delta \rho_\pi(s,a)} =-f^*(s_t,a_t)$ defined in the intent space.
To ensure compatibility in scale with the episode-level advantage and to enable a meaningful interpolation between the two signals, we normalize the potential into a step-level advantage:
\begin{equation}
A_S(s_t,a_t)
=
\frac{-f^*(s_t,a_t)-\text{mean}(-f^*)}
{\operatorname{std}(f^*)} .
\end{equation}

\paragraph{Final advantage formulation.}
We define the final advantage as a weighted interpolation between the episode-level and step-level advantages, to achieve a better bias–variance trade-off:
\begin{equation}
\label{eq:adv_hpo}
\hat A(s^{(i)}_t, a^{(i)}_t)
=
\frac{1}{1+\omega}
\left[
A_E(\tau_i)
+
\omega \cdot A_S(s^{(i)}_t, a^{(i)}_t)
\right],
\end{equation}
where $\omega > 0$ controls the relative contribution of the step-level hindsight signal. By tuning $\omega$, this formulation provides an explicit mechanism to trade off variance reduction against bias, enabling more stable and effective long-horizon optimization.

Let $\mathcal{D}$ denote the training prompt distribution, and let $q \sim \mathcal{D}$ denote a sampled task prompt. Then the policy optimization objective of HPO is:
\begin{align}
\label{eq:hpo}
    \mathcal{J}&_{\mathrm{HPO}}(\theta)=\mathbb{E}_{q \sim \mathcal{D},\left\{\tau^{(i)}\right\}_{i=1}^{G} \sim \pi_{\mathrm{old}}(\cdot \mid q)} \frac{1}{G} \sum_{i=1}^{G} \frac{1}{\left|\tau^{(i)}\right|} \sum_{t=1}^{\left|\tau^{(i)}\right|} \nonumber\\
&\min \left[w_{t}^{(i)}(\theta) \hat{A}_{t}^{(i)}, \operatorname{clip}\left(w_{t}^{(i)}(\theta), 1 \pm \epsilon\right) \hat{A}_{t}^{(i)}\right] ,
\end{align}
where $w_t^{(i)}=\frac{\pi_\theta(y_t^{(i)}|q,y_{<t}^{(i)})}{\pi_\text{old}(y_t^{(i)}|q,y_{<t}^{(i)})}$is the importance sampling ratio between the current policy and the behavior policy.

The overall workflow of \textsc{HPO} is illustrated in Figure~\ref{fig:method_figure}, and the complete algorithm is summarized in Algorithm~\ref{alg:grpo}.
\begin{algorithm}[ht]
  \caption{Hindsight Policy Optimization}
  \label{alg:grpo}
  \begin{algorithmic}
    \STATE {\bfseries Input:} initial policy model $\pi_{\theta_{\text{init}}}$; task prompts $\mathcal{D}$; hyperparameters $\varepsilon, \beta, \mu, \omega$
    \STATE policy model $\pi_\theta \leftarrow \pi_{\theta_{\text{init}}}$
    \FOR{step $= 1, \ldots, M$}
        \STATE Sample a batch $\mathcal{D}_b$ from $\mathcal{D}$
        \STATE Update the old policy model $\pi_{\theta_{\text{old}}} \leftarrow \pi_\theta$
        \STATE Sample outputs $\{o_i\}_{i=1}^{G} \sim \pi_{\theta_{\text{old}}}(\cdot \mid q)$ for each $q \in \mathcal{D}_b$
        \STATE Compute rewards $\{r_i\}_{i=1}^{G}$ for each output $o_i$
        \STATE Get embeddings by encoding $(s,a)\in\{\tau_i\}_{i=1}^G$ with $\Phi$
        \STATE Construct distributions $\rho_\pi$, $\rho^h_\pi$ using Eq.~\eqref{eq:hindsight_intent}
        \STATE Get ${f^*}$ from computing $\mathcal{W}(\rho_\pi,\rho^h_\pi)$
        \STATE Compute $\hat{A}^{(i)}_t$ according to Eq.~\eqref{eq:adv_hpo}
        \FOR{iteration $= 1, \ldots, \mu$}
            \STATE Update the policy $\pi_\theta$ with the objective in Eq.~\eqref{eq:hpo}.
        \ENDFOR
    \ENDFOR
    \STATE {\bfseries Output:} $\pi_\theta$
  \end{algorithmic}
\end{algorithm}

\subsection{Theoretical Analysis}
\label{sec:the_ana}
\begin{lemma}
Let $\rho_\pi$ be a state-action occupancy measure, and define $D=\sup_{(s,a),(s',a')\in \rho_\pi} d\big((s,a),(s',a')\big)$. Let $f^*(s,a)$ denote the $-\frac{\delta W_1(\rho_\pi || \rho^h_\pi)}{\delta \rho_\pi(s,a)}$, and $g(s,a)$ denote the $-\frac{\delta KL(\rho^h_\pi || \rho_\pi)}{\delta \rho_\pi(s,a)}$, then we have
\begin{align}
    &\mathrm{Var}_{\rho_\pi}(f^*(s,a)) \le D^2/4. \\
        &\mathrm{Var}_{\rho_\pi}(g(s,a))
=\int \frac{\rho_\pi^h(s,a)^2}{\rho_\pi(s,a)}\,d(s,a) - 1
= \chi^2(\rho_\pi^h\|\rho_\pi).
\end{align}
where $\chi^2(\cdot\|\cdot)$ denotes the chi-squared divergence.

\end{lemma}

The lemma highlights a sharp contrast between the learning signals induced by Wasserstein and KL. The $f^*(s,a)$ is 1-Lipschitz over the intent space and thus admits a \textit{variance bound} $\mathrm{Var}_{\rho_\pi}(f^*) \le D^2/4$. In contrast, the $g(s,a)=\rho_\pi^h(s,a)/\rho_\pi(s,a)$ has variance characterized by $\chi^2(\rho_\pi^h\|\rho_\pi)$, which can be \textit{arbitrarily large}.

This contrast explains the instability of KL-based signals in large discrete action spaces. The KL gradient depends on the pointwise ratio $\rho_\pi^h(s,a)/\rho_\pi(s,a)$, whose variance is governed by $\chi^2(\rho_\pi^h|\rho_\pi)$ and can become arbitrarily large when $\rho_\pi(s,a)$ is small. Since discrete actions are treated as isolated atoms, no statistical sharing occurs across actions, making reliable estimation of such ratios sample-inefficient, especially in the presence of rare but important events. In contrast, the Wasserstein signal leverages the geometry of the intent space to compare distributions via mass transport, allowing probability to be matched across nearby intents without pointwise division. This yields a smoother learning signal with variance bounded by the metric diameter $D$, providing a principled justification for \textsc{HPO}'s intent-space formulation.

\begin{table*}[t]
\centering
\small
\setlength{\tabcolsep}{4.5pt}
\renewcommand{\arraystretch}{1.05}
\caption{Performance on SearchQA. Results are averaged over 3 random seeds. \textbf{Compared with other baselines, HPO achieves consistent improvements.}}
\resizebox{0.9\textwidth}{!}{
\begin{tabular}{lccccccccc|c}
\toprule
 & \multicolumn{4}{c}{General QA} & \multicolumn{5}{c}{Multi-Hop QA} & \multirow{2}{*}{Avg.} \\
\cmidrule(lr){2-5}\cmidrule(lr){6-10}
Method & NQ & TriviaQA & PopQA & Avg. & HotpotQA & 2wiki & Musique & Bamboogle & Avg. & \\
\midrule

\multicolumn{11}{l}{\textit{Qwen2.5-3B-Instruct}} \\
Base            & 6.4  & 21.0  & 13.3  & 13.6  & 7.6  &  7.4 & 2.7  & 11.2  & 7.2  & 9.9 \\
Search-R1 & $34.1$ & $54.5$ & $37.8$ & $42.1$ & $32.4$ & $31.9$ & $10.3$ & $26.4$ & $25.3$ & $32.5$ \\
ZeroSearch & $41.4$ & $57.4$ & $44.8$ & $47.9$ & $27.4$ & $30.0$ & $9.8$ & $11.1$ & $19.6$ & $31.7$ \\

PPO    & $36.2_{\pm 1.3}$ & $57.5_{\pm 2.8}$ & $40.7_{\pm 2.8}$ & $44.8_{\pm 1.9}$ & $32.8_{\pm 2.4}$ & $33.0_{\pm 1.3}$ & $11.6_{\pm 1.6}$ & $30.4_{\pm 2.1}$ & $27.0_{\pm 1.6}$ & $34.6_{\pm 1.8}$ \\
GRPO  & $37.2_{\pm 2.4}$ & $56.5_{\pm 3.4}$ & $42.4_{\pm 2.4}$ & $45.4_{\pm 2.0}$ & $32.0_{\pm 3.3}$ & $33.8_{\pm 2.0}$ & $12.3_{\pm 1.0}$ & $28.4_{\pm 1.5}$ & $26.6_{\pm 1.2}$ & $34.7_{\pm 1.5}$ \\
\rowcolor{gray!15}
\textbf{HPO$_{\text{on}}$}
& $\underline{41.2}_{\pm 1.4}$ & $\underline{59.8}_{\pm 2.4}$ & $\underline{44.0}_{\pm 1.8}$ & $\underline{48.3}_{\pm 1.2}$ & $\underline{36.2}_{\pm 1.8}$ & $\underline{34.2}_{\pm 2.2}$ & $\underline{14.4}_{\pm 1.4}$ & $\underline{32.7}_{\pm 3.7}$ & $\underline{29.4}_{\pm 2.0}$ & $\underline{37.5}_{\pm 1.6}$ \\
\rowcolor{gray!15}
\textbf{HPO$_{\text{off}}$}
& $\mathbf{43.8}_{\pm 0.1}$ & $\mathbf{61.1}_{\pm 1.1}$ & $\mathbf{45.1}_{\pm 0.6}$ & $\mathbf{50.0}_{\pm 0.2}$ & $\mathbf{37.6}_{\pm 0.7}$ & $\mathbf{35.1}_{\pm 0.9}$ & $\mathbf{15.5}_{\pm 1.1}$ & $\mathbf{36.0}_{\pm 1.2}$ & $\mathbf{31.1}_{\pm 1.2}$ & $\mathbf{39.2}_{\pm 0.8}$ \\
\midrule

\multicolumn{11}{l}{\textit{Qwen2.5-7B-Instruct}} \\
Base            & 11.9 & 31.1 & 20.3 & 21.1 & 16.5  & 10.9  & 7.0  & 21.6 & 14.0  & 17.0 \\
Search-R1 & $39.3$ & $61.0$ & $39.7$ & $46.7$ & $37.0$ & $\underline{41.4}$ & $14.6$ & $36.8$ & $32.5$ & $38.5$ \\
ZeroSearch & $43.6$ & $61.8$ & $\mathbf{51.5}$ & $52.3$ & $34.6$ & $35.2$ & $18.4$ & $27.8$ & $29.0$ & $39.1$ \\

PPO  & $41.2_{\pm 2.0}$ & $64.0_{\pm 1.7}$ & $42.5_{\pm 2.4}$ & $49.2_{\pm 1.5}$ & $39.5_{\pm 2.1}$ & $40.0_{\pm 1.6}$ & $18.8_{\pm 1.0}$ & $37.6_{\pm 2.7}$ & $34.0_{\pm 1.8}$ & $40.5_{\pm 1.5}$ \\
GRPO               & $42.2_{\pm 4.1}$ & $62.8_{\pm 2.5}$ & $42.5_{\pm 4.0}$ & $49.2_{\pm 3.5}$ & $38.5_{\pm 4.1}$ & $38.7_{\pm 3.2}$ & $17.1_{\pm 3.8}$ & $34.8_{\pm 4.0}$ & $32.3_{\pm 3.8}$ & $39.5_{\pm 3.6}$ \\
\rowcolor{gray!15}
\textbf{HPO$_{\text{on}}$}
& $\underline{48.1}_{\pm 1.7}$ & $\underline{66.4}_{\pm 0.6}$ & $46.3_{\pm 1.8}$ & $\underline{53.6}_{\pm 1.2}$ & $\mathbf{44.4}_{\pm 2.8}$ & $\mathbf{42.3}_{\pm 2.2}$ & $20.3_{\pm 1.0}$ & $\mathbf{40.8}_{\pm 2.8}$ & $\mathbf{37.0}_{\pm 1.5}$ & $\mathbf{44.1}_{\pm 1.3}$ \\
\rowcolor{gray!15}
\textbf{HPO$_{\text{off}}$}
& $\mathbf{48.2}_{\pm 0.6}$ & $\mathbf{67.1}_{\pm 1.1}$ & $\underline{47.9}_{\pm 0.8}$ & $\mathbf{54.4}_{\pm 0.4}$ & $\underline{43.2}_{\pm 0.1}$ & $40.5_{\pm 0.4}$ & $\underline{19.3}_{\pm 0.4}$ & $\mathbf{41.3}_{\pm 1.1}$ & $\underline{36.1}_{\pm 0.2}$ & $\underline{43.9}_{\pm 0.3}$ \\
\bottomrule
\end{tabular}
}

\label{tab:search_result}
\end{table*}

\begin{table}[h]
\centering
\small
\setlength{\tabcolsep}{6pt}
\renewcommand{\arraystretch}{1.05}
\caption{The results on Textcraft benchmark.}
\resizebox{0.75\linewidth}{!}{
\begin{tabular}{lcc}
\toprule
Method
& \multicolumn{2}{c}{Success Rate (\%)} \\
\cmidrule(lr){2-3}
& {\textit{Qwen2.5-3B-Ins}}
& {\textit{Qwen2.5-7B-Ins}} \\
\midrule
Base & 14.0 & 42.0 \\
PPO  & 76.0 & 89.0 \\
GRPO & 81.0 & 87.0 \\
\rowcolor{gray!15}
\textbf{HPO$_{\text{on}}$} & \textbf{84.0} & \textbf{93.0} \\
\rowcolor{gray!15}
\textbf{HPO$_{\text{off}}$} & \textbf{86.0} & \textbf{93.0} \\
\bottomrule
\end{tabular}
}
\label{tab:textcraft}
\end{table}

\section{Experiments}
\subsection{Setup}
\paragraph{Tasks and Benchmarks.} We evaluate HPO on two long-horizon interactive tasks: SearchQA and TextCraft~\cite{prasad2024adapt}. For SearchQA, we consider seven benchmarks, including single-hop QA (NQ~\cite{kwiatkowski2019natural}, TriviaQA~\cite{joshi2017triviaqa}, and PopQA~\cite{mallen2022not}) and multi-hop QA  (HotpotQA~\cite{yang2018hotpotqa}, 2Wiki~\cite{ho2020constructing}, Musique~\cite{trivedi2023interleaving}, and Bamboogle~\cite{press2023measuring}). For TextCraft, we report the average success rate (\%) for each query. Full settings details are provided in Appendix~\ref{sec:exp_dt}.

\paragraph{Models and Baselines.} We use Qwen2.5-3B/7B-Instruct~\cite{qwen2025qwen25technicalreport} as our base models. For HPO, we use Qwen3-Embedding-0.6B as encoder. We compare our approach against a diverse set of strong baselines. Specifically, we include \textbf{PPO}, a widely adopted actor-critic method that relies on an additional value model, as well as  critic-free methods such as \textbf{GRPO}, which estimate advantages over groups of trajectories. For the SearchQA task, we further evaluate recent competitive models trained with different RL algorithms, including \textbf{Search-R1}, \textbf{ZeroSearch}. Full training settings and hyperparameter details are provided in Appendix~\ref{sec:exp_dt}.

\subsection{Main Result}

\paragraph{HPO significantly improves policy performance, especially on multi-hop tasks.} As shown in Table~\ref{tab:search_result}, HPO improves average performance by 4\%-7\% over other baselines. Moreover, the improvements are especially pronounced on multi-hop benchmarks such as HotpotQA and Bamboogle. We analyze that these gains likely stem from HPO’s fine-grained, step-level advantages, which enable more accurate credit assignment by identifying and reinforcing truly valuable actions; this in turn translates into higher final answer accuracy on multi-hop questions.

\paragraph{Offline HPO achieves better performance than  online HPO.}
As shown in Table~\ref{tab:search_result}, offline HPO is more stable, with smaller performance variance across benchmarks. Although online HPO can construct an on-policy hindsight distribution from newly collected rollouts, the number of successful samples used for constructing hindsight distribution fluctuates under stochastic sampling, leading to unstable estimation. Offline HPO instead constructs the hindsight from a fixed offline dataset, and it further mitigates the issue of sparse rewards: when no correct solution is found, both GRPO and online HPO cannot update policy effectively. In contrast, offline HPO can still update by constructing the hindsight from a fixed set of successful trajectories in the offline data, thus providing informative step-level advantages.

\paragraph{HPO consistently improves performance across long-horizon agent tasks.}
Table~\ref{tab:textcraft} reports results on the TextCraft benchmark. HPO achieves consistent improvements in average success rate over all baselines, demonstrating that HPO is effective across different task domains and interaction dynamics, highlighting its robustness and broad applicability to long-horizon agent settings.

\subsection{Training Dynamics}

Figure~\ref{fig:dynamics} shows the training reward dynamics on Qwen2.5-7B. GRPO exhibits unstable long-horizon training with abrupt reward collapses, as it fails to distinguish successful actions from failed ones within trajectories. PPO achieves more stable training by introducing a critic to identify valuable actions, but converges more slowly because the critic requires substantial warm-up to produce reliable value estimates. In contrast, HPO achieves fast and stable convergence by identifying valuable actions without training an additional critic model. Notably, offline HPO demonstrates the fastest early-stage reward growth, which can be attributed to its step-level advantage in identifying correct actions within failed trajectories, thereby enabling meaningful policy updates even in the absence of fully correct samples.

\begin{figure}[h]
    \centering
    \includegraphics[width=0.9\linewidth]{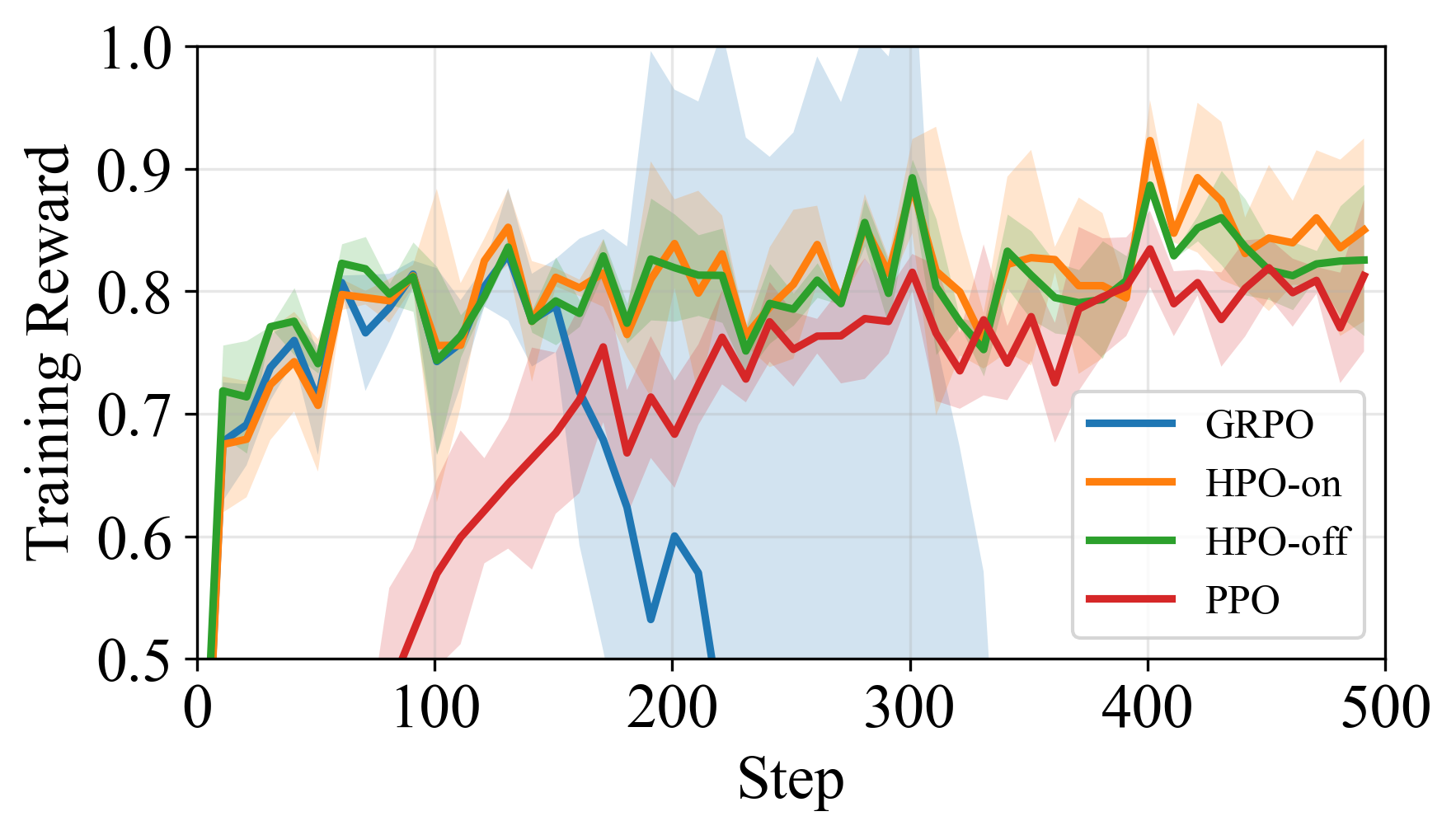}
    \caption{The reward dynamics of HPO and other baseline methods with Qwen2.5-7B over SearchQA, where curves and uncertainty ranges are computed using the mean and standard deviation over three random seeds, respectively. \textbf{GRPO exhibits unstable training, PPO converges more slowly, while HPO achieves fast and stable convergence.}}
    \label{fig:dynamics}
\end{figure}

\subsection{Interpretability Analysis of $A_S$}

\paragraph{$A_S$ alone is sufficient to drive effective policy updates.} To isolate the contribution of $A_E$, we ablate it and re-run off-policy HPO over Qwen2.5-3B. In this setting, policy optimization is driven solely by the learning signal from $A_S$, while outcome rewards are computed only for monitoring and reporting training dynamics. As shown in Figure~\ref{fig:onlyk}, optimizing with $A_S$ alone yields stable and effective policy updates: the monitored training reward increases steadily throughout optimization, without exhibiting reward collapse. The final model is only marginally worse than the standard outcome-reward–optimized model, suggesting that the bias introduced by $A_S$ is limited and practically acceptable. Overall, these results highlight the practical utility of $A_S$, particularly for open-ended tasks where outcome rewards are difficult to specify, or for online settings where outcome evaluation is costly, since $A_S$ can be learned from offline data that has already been evaluated. The \textbf{case study} in the Appendix~\ref{app:case} further confirms that \(A_S\) reliably identifies valuable actions and provides sufficiently informative step-level signals for policy optimization.

\begin{figure}[h]
    \centering
    \includegraphics[width=0.9\linewidth]{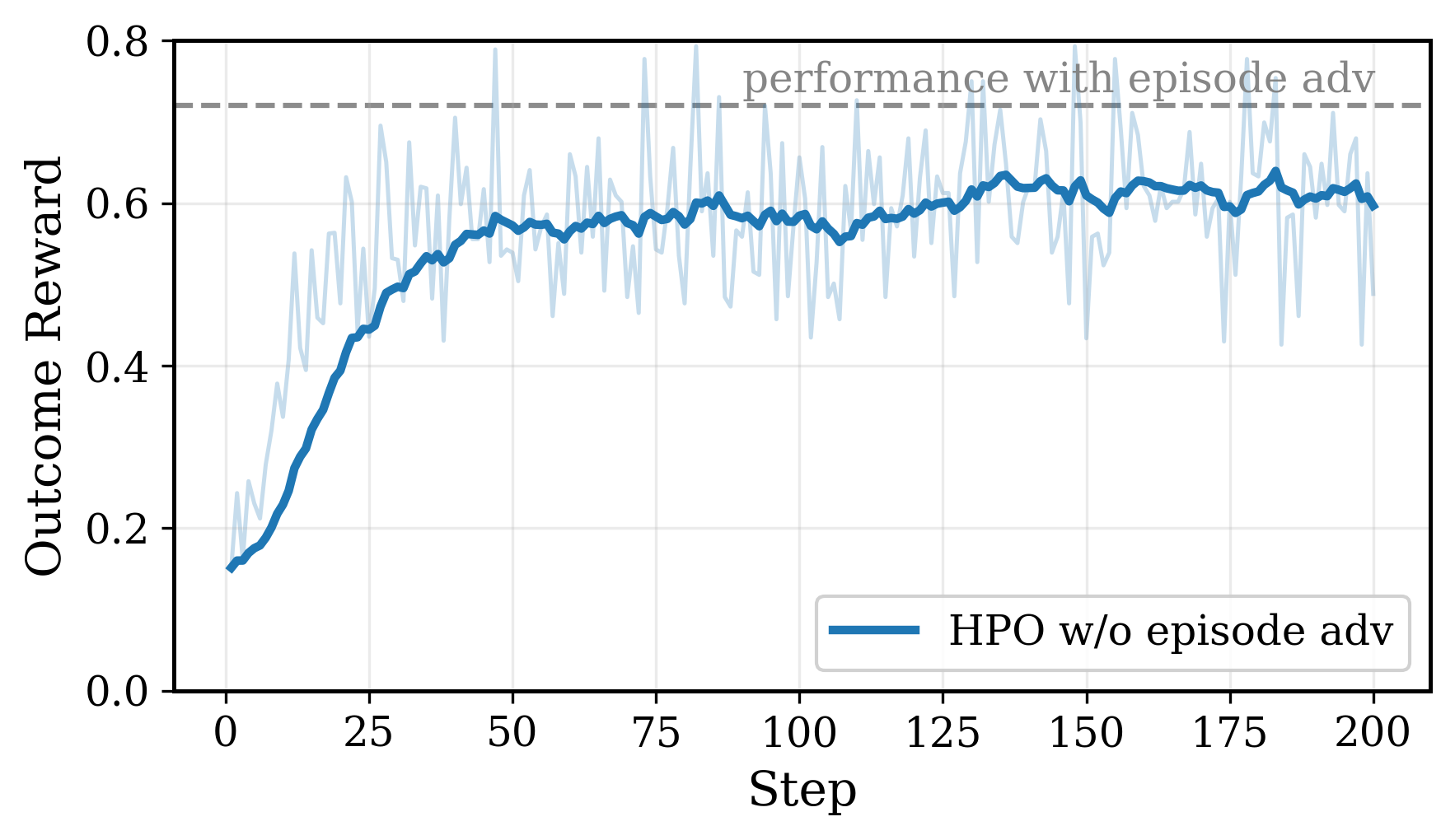}
    \caption{Training reward dynamics of HPO using only step-level advantages with Qwen2.5-3B. The dark blue curve is EMA-smoothed from the raw data. \textbf{Without episode-level rewards, training remains stable and effective, indicating that step-level advantages provide a reliable learning signal.}}
    \label{fig:onlyk}
\end{figure}

\subsection{Efficiency Analysis}
We analyze the computational budget of HPO. HPO shares the same core architecture as GRPO, including multi-turn rollouts, computation of old and reference probabilities, and policy updates. The primary additions introduced by HPO are the step-relative advantage estimation components. To evaluate cost, we train a SearchQA LLM agent using Qwen2.5-7B-Instruct with Qwen3-0.6B-Embedding as the encoder, and record a per-iteration time breakdown.
\begin{figure}[h]
    \centering

    \includegraphics[width=0.9\linewidth]{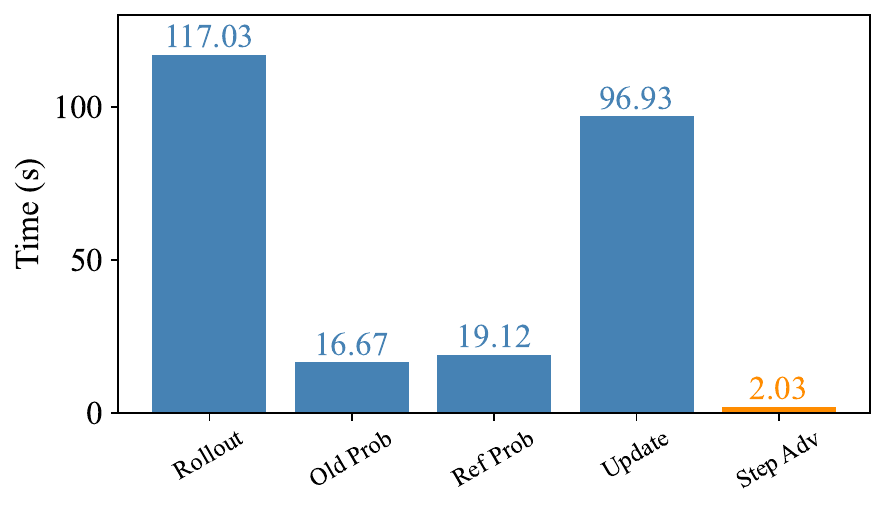}
    \caption{Per-iteration training time breakdown of HPO with Qwen2.5-7B. Blue bars denote components shared with GRPO, while orange bars denote HPO-specific additions. \textbf{The added overhead in HPO is negligible ($<0.8\%$).}}
    \label{fig:efficiency Analysis}
\end{figure}

As shown in Figure~\ref{fig:efficiency Analysis}, the added components introduce negligible overhead. Compared with GRPO, HPO requires only an additional 2.03s to compute the step-level advantage, accounting for approximately 0.8\% of the total time per step, which demonstrates that HPO achieves computational efficiency comparable to that of GRPO.

\section{Related Work}
\paragraph{LLMs as decision-making agents.}
Large language models (LLMs) have been increasingly used as autonomous agents for reasoning, planning, and action execution across diverse domains, including code generation~\cite{zhang2024codeagent}, smart device control~\cite{zhang-zhang-2024-look,gurreal,he2024webvoyager,hong2024cogagent}, and interactive games~\cite{abdulhailmrl,prasad2024adapt}. Early approaches primarily relied on prompting-based frameworks such as ReAct~\cite{yao2022react} and Self-Refine~\cite{madaan2023self}, which depend on strong proprietary models (e.g., OpenAI o3) and do not endow models with intrinsic agentic capabilities. More recent work shifts toward supervised fine-tuning (SFT)~\cite{chen2023fireact} and reinforcement learning (RL)~\cite{zhai2024fine}, allowing agents to acquire behaviors directly through data or environment interaction rather than handcrafted prompts.

\paragraph{Reinforcement learning for LLM agents.}
Reinforcement learning has emerged as a key post-training paradigm for LLMs, supporting both improved reasoning~\cite{guo2025deepseek,jaech2024openai,pennino2025reasoning} and interactive decision-making~\cite{xi2025agentgym,peiyuan2024agile}. Common algorithms include PPO~\cite{schulman2017proximal}, GRPO~\cite{shao2024deepseekmath}, and REINFORCE++~\cite{hu2025reinforce++}. However, many representative works (e.g., DeepSeek-R1) focus on single-turn settings and do not apply to long-horizon interaction. Although recent studies explore multi-step RL for agent training~\cite{cao2025skyrl,jin2025search,qiwebrl,wang2025ragen}, they often face challenges in scalability and optimization stability when applied to complex and diverse environments~\cite{jin2025search,xue2025simpletir,sun2025zerosearch}.

\section{Conclusion}
This work proposes a novel policy gradient estimator that compares the current policy with the hindsight distribution in an intent space, yielding low-variance and smooth learning signals that substantially improve the stability and efficiency of long-horizon training, and offering a new pathway toward stable critic-free reinforcement learning. Future work will explore more fine-grained constructions of the intent space and extend this framework to more complex multimodal and open-ended environments.

\section*{Acknowledgments}
Thanks for the kind suggestions and support from Huawei Large Model Data Technology Lab.

\section*{Impact Statement}
This paper presents work whose goal is to advance the field of Machine
Learning. There are many potential societal consequences of our work, none which we feel must be specifically highlighted here.
\bibliography{example_paper}
\bibliographystyle{icml2026}

\newpage
\appendix
\onecolumn
\section*{Appendix}

\section{Experimental Details}
\label{sec:exp_dt}
\subsection{Task Description}
\paragraph{SearchQA.} The deep search senario features a search engine–based environment equipped with specialized tools and APIs supporting the interaction with search engines. These APIs enable agents to dynamically generate search queries during the reasoning process, retrieve relevant information from external sources, and incorporate the retrieved information into subsequent reasoning steps. This setting allows agents to engage in complex reasoning processes that involve iterative searching and information integration, thereby enhancing their capability to solve intricate problems where external knowledge is essential.

\paragraph{TextCraft.} TextCraft is a text-based game environment mirroring Minecraft. The APIs in TextCraft include crafting, inventory management, and dynamic narrative generation. These APIs allow agents to execute predefined crafting recipes, manipulate inventory contents, navigate virtual spaces.

\subsection{Experimental Settings.}
\paragraph{Settings for SearchQA.} The maximum prompt length is 1024 tokens, and the maximum response length is 512 tokens. The max turn is set to 4. The learning rate is $1 \times 10^{-6}$ for the actor. The reward consists of an outcome reward and a format reward, weighted in a ratio of 8:2. We set the train batch size to 64 and use a group size of 8. Rollout and validation temperatures are set to 1.0 and 0.0, respectively. The mini-batch size is 32, and the KL-divergence loss coefficient is set to 0.001. We use E5 as the retriever. The weighting coefficient $\omega$ is set to $0.5$. For the training data, we use a mixture of NQ and HotpotQA. In addition, for offline HPO, we use 8 offline correct trajectories per query. We use Qwen3-Embedding-0.6B as encoder. During training, we save a model checkpoint every 50 steps; for the main experimental results, we test these checkpoints and report the best testing performance.

\paragraph{Settings for TextCraft.} The maximum prompt length is 1024 tokens, and the maximum response length is 512 tokens. The max turn is set to 30. The learning rate is $1 \times 10^{-6}$ for the actor. We adopt a outcome-based reward, assigning a reward of 1 for success and 0 for failure. We set the train batch size to 64 and use a group size of 8. Rollout and validation temperatures are set to 1.0 and 0.0, respectively. The mini-batch size is 32, and the KL-divergence loss coefficient is set to 0.001. The weighting coefficient $\omega$ is set to $0.5$. For offline HPO, we use 8 offline correct trajectories per query. We use Qwen3-Embedding-0.6B as encoder. During training, we save a model checkpoint every 50 steps; for the main experimental results, we test these checkpoints and report the best testing performance.

\section{General Hindsight Distribution}
\label{app:general_hindsight}

We first revisit the binary-reward definition used in the main text. Let
\(R_t=\sum_{t'=t}^{\infty}r_{t'}\) denote the future return from step \(t\). In the binary setting, the hindsight distribution is defined by conditioning on successful future outcomes:
\begin{equation}
\rho_\pi^h(s,a)
=
\sum_{t=0}^{\infty}
\mathbb{P}_\pi\!\left((s_t,a_t)=(s,a)\mid R_t=1\right).
\end{equation}
By Bayes' rule, this can be rewritten as
\begin{align}
\rho_\pi^h(s,a)
&=
\frac{1}{Z}
\sum_{t=0}^{\infty}
\mathbb{P}_\pi\!\left((s_t,a_t)=(s,a)\right)
\mathbb{P}_\pi\!\left(R_t=1\mid s_t=s,a_t=a\right) \\
&=
\frac{\rho_\pi(s,a)Q_\pi(s,a)}{Z},
\end{align}
where \(Q_\pi(s,a)=\mathbb{P}_\pi(R_t=1\mid s_t=s,a_t=a)\) in the binary-reward case and \(Z\) is the normalizing constant.

This form suggests a direct extension beyond binary rewards: instead of conditioning on the event \(R_t=1\), we construct the hindsight distribution by reweighting probability mass according to realized future returns. Under continuous or dense rewards, the hindsight distribution should therefore take the following form:
\begin{equation}
\rho_\pi^h(s,a)
=
\frac{1}{Z}
\sum_{t=0}^{\infty}
\mathbb{E}_{\pi}
\!\left[
R_t\mathbf{1}\!\left((s_t,a_t)=(s,a)\right)
\right],
\end{equation}
where \(Z\) normalizes the distribution. Equivalently, this can be written as \(\rho_\pi^h(s,a)=\rho_\pi(s,a)Q_\pi(s,a)/Z\), where \(Q_\pi(s,a)=\mathbb{E}_\pi[R_t\mid s_t=s,a_t=a]\). If returns can be negative or have different scales, one can use a nonnegative normalized or shifted return as the weighting function.

In summary, the two cases are unified by the same reweighting view. \textbf{Binary rewards (0/1):} the distribution is obtained via rejection sampling over trajectories, retaining only those with return equal to 1. \textbf{Continuous rewards:} the distribution is constructed by reweighting probability mass according to trajectory returns.

\section{Embedding Model Ablation}
\label{app:embedding_ablation}

HPO constructs intent distributions in a semantic embedding space, so a natural question is whether its step-level advantage estimates are sensitive to the particular embedding model used to represent state--action pairs. To examine this, we conduct an ablation over the choice of encoder. In our main experiments, we use \textbf{Qwen3-Embedding-0.6B} as the default encoder. We compare it with two alternatives: \textbf{Qwen3-Embedding-4B}, which has a larger model scale within the same embedding family, and \textbf{all-MiniLM-L12-v2}, which provides a smaller encoder with a different architecture.

For each encoder, we recompute the HPO step-level advantages on the same set of 256 samples. We then compare the resulting advantage rankings using Kendall's $\tau$, which measures whether different encoders assign consistent relative importance to the same steps. Let m1, m2, and m3 denote Qwen3-Embedding-0.6B, Qwen3-Embedding-4B, and all-MiniLM-L12-v2, respectively. The rank correlations are highly significant across encoder choices: m1--m2 yields $p < 2.2\times 10^{-308}$, and m1--m3 yields $p = 6.8\times 10^{-212}$.

These results show that the advantage estimates produced by HPO are strongly aligned across embedding models, including both a larger encoder from the same family and an encoder with a different architecture. This suggests that HPO mainly relies on stable semantic structure in the embedding space rather than idiosyncrasies of a particular encoder, indicating that the method is not sensitive to the specific embedding model choice.

\section{Proof}
\subsection{Proof of Lemma 3.1}
\begin{lemma}[Variance Decomposition]

Consider the REINFORCE gradient estimator $\hat{G}_b$ with the baseline $b$,
\begin{equation}
\hat{G}_b \;=\; g(a)\,\big(\hat Q(s,a)-b(s)\big),
\end{equation}
where $a \sim \pi(\cdot\mid s)$ and $g(a)=\nabla_\theta \log \pi_\theta(a\mid s)$.

Let $\|\cdot\|$ denotes the Euclidean norm.  Then its conditional variance admits the decomposition
\begin{align}
\mathrm{Var}(\hat{G}_b \mid s)\!=\!\mathbb{E}\!\left[\|g(a)\|^2\mid s\right]
\!\cdot\!
\mathrm{Var}\!\left(\hat Q(s,a)-b(s)\mid s\right)
\end{align}
\end{lemma}
\begin{proof}
Condition on the state $s$ and write
\[
\hat G_b = g(a)\,X,\qquad X := \hat Q(s,a)-b(s),\qquad a\sim \pi(\cdot\mid s).
\]
Under the standard simplifying assumption that $g(a)$ and $\hat Q(s,a)$ are (conditionally) uncorrelated given $s$, we have
\[
\text{Var}(\hat G_b\mid s)=\mathbb{E}\!\left[\|g(a)\|^2 X^2 \mid s\right]-\left\|\mathbb{E}[g(a)X\mid s]\right\|^2.
\]
Using conditional independence,
\[
\mathbb{E}\!\left[\|g(a)\|^2 X^2 \mid s\right]=\mathbb{E}\!\left[\|g(a)\|^2\mid s\right]\mathbb{E}\!\left[X^2\mid s\right],
\qquad
\mathbb{E}[g(a)X\mid s]=\mathbb{E}[g(a)\mid s]\mathbb{E}[X\mid s].
\]
Moreover, for score functions one typically has $\mathbb{E}[g(a)\mid s]=\mathbb{E}[\nabla_\theta \log\pi_\theta(a\mid s)\mid s]=0$.
Therefore,
\[
\text{Var}(\hat G_b\mid s)=\mathbb{E}\!\left[\|g(a)\|^2\mid s\right]\mathbb{E}\!\left[X^2\mid s\right].
\]
Finally, if the baseline is chosen as the conditional mean $b(s)=\mathbb{E}[\hat Q(s,a)\mid s]$, then $\mathbb{E}[X\mid s]=0$ and thus $\mathbb{E}[X^2\mid s]=\text{Var}(X\mid s)=\text{Var}(\hat Q(s,a)-b(s)\mid s)$, yielding
\[
\text{Var}(\hat G_b\mid s)=\mathbb{E}\!\left[\|g(a)\|^2\mid s\right]\text{Var}\!\left(\hat Q(s,a)-b(s)\mid s\right),
\]
\end{proof}

\subsection{Proof of Theorem 3.3}
\begin{theorem}[Optimal Baseline and Excess Variance]

Among all scalar baselines $b(s)$ independent of the sampled action $a$, the conditional variance
$\mathrm{Var}(\hat G_b \mid s)$ is minimized by
\begin{equation}
b^*(s) =
\frac{\mathbb{E}\!\left[\|g(a)\|^2 \hat Q(s,a)\mid s\right]}
{\mathbb{E}\!\left[\|g(a)\|^2\mid s\right]}.
\end{equation}
Moreover, for any alternative baseline $\tilde b(s)$, the increase in variance admits the exact expression
\begin{align}
    \mathrm{Var}(\hat{G}_{\tilde b} \mid s)& \,-\,  \mathrm{Var}(\hat{G}_{b^*} \mid s), \,=\,   \nonumber\\
&\mathbb{E}\!\left[\,\|g(a)\|^2\mid s\right](\tilde b(s)-b^*(s))^2.
\end{align}
\end{theorem}

\begin{proof}
Fix $s$ and write $X=\hat Q(s,a)$.
For any scalar baseline $b=b(s)$ independent of $a$,
\[
\text{Var}(\hat G_b\mid s)
=\mathbb{E}\!\left[\|g(a)\|^2(X-b)^2\mid s\right]
-\left\|\mathbb{E}[g(a)X\mid s]\right\|^2,
\]
where we use $\mathbb{E}[g(a)\mid s]=0$, so subtracting an action-independent baseline does not change the conditional mean. Since the second term does not depend on $b$, the variance is minimized by
\[
b^*(s)=\arg\min_b \mathbb{E}\!\left[\|g(a)\|^2(X-b)^2\mid s\right]
=\frac{\mathbb{E}\!\left[\|g(a)\|^2X\mid s\right]}{\mathbb{E}\!\left[\|g(a)\|^2\mid s\right]}.
\]
For any alternative $\tilde b(s)$,
\begin{align*}
&\mathbb{E}\!\left[\|g(a)\|^2(X-\tilde b)^2\mid s\right] \\
&= \mathbb{E}\!\left[\|g(a)\|^2(X-b^*+b^*-\tilde b)^2\mid s\right] \\
&= \mathbb{E}\!\left[\|g(a)\|^2(X-b^*)^2\mid s\right]
-2(\tilde b-b^*)\,\mathbb{E}\!\left[\|g(a)\|^2(X-b^*)\mid s\right]+\mathbb{E}\!\left[\|g(a)\|^2\mid s\right](\tilde b-b^*)^2.
\end{align*}
Since $b^*=\mathbb{E}[\|g(a)\|^2X\mid s]/\mathbb{E}[\|g(a)\|^2\mid s]$, the cross term vanishes,
$\mathbb{E}[\|g(a)\|^2(X-b^*)\mid s]=0$, and therefore
\[
\mathbb{E}\!\left[\|g(a)\|^2(X-\tilde b)^2\mid s\right]
-\mathbb{E}\!\left[\|g(a)\|^2(X-b^*)^2\mid s\right]
=\mathbb{E}\!\left[\|g(a)\|^2\mid s\right](\tilde b-b^*)^2.
\]

which immediately gives
\[
\text{Var}(\hat G_{\tilde b}\mid s)-\text{Var}(\hat G_{b^*}\mid s)
=\mathbb{E}[\|g(a)\|^2\mid s]\,
(\tilde b(s)-b^*(s))^2.
\]
\end{proof}

\subsection{Proof of Lemma 4.2}
\label{app:proof_lem42}
\begin{lemma}
Consider the policy objective $\mathcal{J}(\theta)=\mathbb{E}_{\tau \sim \pi}[\sum_{t\ge 0}r_t]$. Its policy gradient admits the following two equivalent forms:
\begin{align}
    \nabla_\theta \mathcal{J}(\theta)& = \mathbb{E}_{(s,a)\sim\rho_{\pi}}[Q_\pi(s,a) \nabla_\theta\log \pi_\theta (a\mid s) ] \\
    & =\mathbb{E}_{(s,a)\sim\rho_\pi}[-\frac{\delta KL(\rho^h_\pi || \rho_\pi)}{\delta \rho_\pi(s,a)}\nabla_\theta \log \pi_\theta (a\mid s)]
\end{align}
where $-\frac{\delta KL(\rho^h_\pi || \rho_\pi)}{\delta \rho_\pi(s,a)}=\frac{\rho^h_\pi(s,a)}{\rho_\pi(s,a)}$ denotes the pointwise variational derivative of $KL(\rho^h_\pi || \rho_\pi)$ with respect to $\rho_\pi(s,a)$.

\end{lemma}
\begin{proof}
For $\mathcal{J}(\theta)=\mathbb{E}_{\tau \sim \pi}[\sum_{t\ge 0}r_t]$, by the policy gradient theorem, its gradient admits the following equivalent form
\begin{equation}
    \nabla_\theta \mathcal{J}(\theta)=\mathbb{E}_{(s,a)\sim\rho_\pi}[Q_\pi(s,a)\nabla_\theta \log \pi_\theta (a\mid s) ]
\end{equation}
Writing the expectation explicitly, we obtain
\begin{equation}
\nabla_\theta \mathcal{J}(\theta)=\sum_{s,a} \rho_\pi(s,a)Q_\pi(s,a)\nabla_\theta \log \pi_\theta (a\mid s)
\end{equation}
Let $Z=\sum_{s,a}\rho_\pi(s,a)Q_\pi(s,a)$. Define $\rho_\pi^h(s,a)=\rho_\pi(s,a)Q_\pi(s,a)/Z$. Then the above becomes
\begin{align}
    \nabla_\theta \mathcal{J}(\theta)&=\sum_{s,a} \rho_\pi(s,a)Q_\pi(s,a)\nabla_\theta \log \pi_\theta (a\mid s) \\
    & =Z\sum_{s,a}\rho_\pi(s,a) \frac{\rho^h_\pi(s,a)}{\rho_\pi(s,a)}\nabla_\theta \log \pi_\theta (a\mid s)  \\
    &=Z\,\mathbb{E}_{(s,a)\sim\rho_\pi}[\frac{\rho^h_\pi(s,a)}{\rho_\pi(s,a)}\nabla_\theta \log \pi_\theta (a\mid s) ]
\end{align}
where the ratio \(\frac{\rho^h_\pi(s,a)}{\rho_\pi(s,a)}\) can be interpreted as the pointwise variational derivative of \(KL(\rho^h_\pi || \rho_\pi)\) with respect to \(\rho_\pi\), i.e.,\(
-\frac{\delta KL(\rho^h_\pi || \rho_\pi)}{\delta \rho_\pi(s,a)}
\).

Since \(Z\) is a positive constant, it does not affect the direction of the gradient. Therefore, the policy gradient admits the following two equivalent forms:
\begin{align}
    \nabla_\theta \mathcal{J}(\theta)& = \mathbb{E}_{(s,a)\sim\rho_{\pi}}[Q_\pi(s,a) \nabla_\theta\log \pi_\theta (a\mid s) ] \\
    & =\mathbb{E}_{(s,a)\sim\rho_\pi}[-\frac{\delta KL(\rho^h_\pi || \rho_\pi)}{\delta \rho_\pi(s,a)}\nabla_\theta \log \pi_\theta (a\mid s)]
\end{align}
\end{proof}

\subsection{Proof of Lemma 4.3}
\begin{lemma}
Define $D=\sup_{(s,a),(s',a')\in \rho_\pi} d\big((s,a),(s',a')\big)$. Let $f^*(s,a)$ denote the $-\frac{\delta W_1(\rho_\pi || \rho^h_\pi)}{\delta \rho_\pi(s,a)}$, and $g(s,a)$ denote the $-\frac{\delta KL(\rho^h_\pi || \rho_\pi)}{\delta \rho_\pi(s,a)}$, then
\begin{equation}
    \mathrm{Var}_{\rho_\pi}(f^*(s,a)) \le D^2/4.
\end{equation}
and
\begin{equation}
    \mathrm{Var}_{\rho_\pi}(g(s,a))
=\int \frac{\rho_\pi^h(s,a)^2}{\rho_\pi(s,a)}\,d(s,a) - 1
= \chi^2(\rho_\pi^h\|\rho_\pi).
\end{equation}
\end{lemma}

\begin{proof}
For the $W_1$ term, recall the Kantorovich--Rubinstein dual:
\begin{equation}
W_1(\rho_\pi^h,\rho_\pi)=\sup_{\|f\|_{\mathrm{Lip}}\le 1}\ \mathbb{E}_{x\sim \rho_\pi^h}[f(x)]-\mathbb{E}_{x\sim \rho_\pi}[f(x)],
\quad x=(s,a).
\end{equation}

Let $f^*$ be the Kantorovich potential, which is $1$-Lipschitz on the support of $\rho_\pi$ (up to an additive constant). By definition of the diameter $D=\sup_{x,x'\in \mathrm{supp}(\rho_\pi)} d(x,x')$, any $1$-Lipschitz function satisfies, for all $x,x'\in \mathrm{supp}(\rho_\pi)$, $|f^*(x)-f^*(x')|\le d(x,x')\le D$,

Therefore, by Popoviciu's inequality on variances,
\begin{equation}
    \mathrm{Var}_{\rho_\pi}(f^*(x))\le \frac{(\sup f^*-\inf f^*)^2}{4}\le \frac{D^2}{4}.
\end{equation}

For the $KL$ term, by the stated definition,
\begin{equation}
    g(x) \;=\; -\frac{\delta KL(\rho_\pi^h\|\rho_\pi)}{\delta \rho_\pi(x)} \;=\; \frac{\rho_\pi^h(x)}{\rho_\pi(x)},
\end{equation}

assuming $\rho_\pi(x)>0$ whenever $\rho_\pi^h(x)>0$. Then
\begin{equation}
    \mathbb{E}_{x\sim \rho_\pi}[g(x)] = \int \rho_\pi(x)\frac{\rho_\pi^h(x)}{\rho_\pi(x)}\,dx = \int \rho_\pi^h(x)\,dx = 1,
\end{equation}
and
\begin{equation}
\mathrm{Var}_{\rho_\pi}(g(x))
= \mathbb{E}_{\rho_\pi}[g(x)^2]-\Big(\mathbb{E}_{\rho_\pi}[g(x)]\Big)^2
= \int \rho_\pi(x)\Big(\frac{\rho_\pi^h(x)}{\rho_\pi(x)}\Big)^2\,dx - 1
= \int \frac{\rho_\pi^h(x)^2}{\rho_\pi(x)}\,dx - 1.
\end{equation}
Finally, noting that
\begin{equation}
\chi^2(\rho_\pi^h\|\rho_\pi)=\int \frac{(\rho_\pi^h(x)-\rho_\pi(x))^2}{\rho_\pi(x)}\,dx
=\int \frac{\rho_\pi^h(x)^2}{\rho_\pi(x)}\,dx -2\int \rho_\pi^h(x)\,dx + \int \rho_\pi(x)\,dx
=\int \frac{\rho_\pi^h(x)^2}{\rho_\pi(x)}\,dx -1,
\end{equation}
\end{proof}

\section{Additional Analysis}
\subsection{Analysis of the weight $\omega$}
To further investigate the effect of the step-level advantage weight $\omega$, multiple repeated HPO-off experiments were conducted with different $\omega$ settings. Figure~\ref{fig:omega} illustrates the training reward trajectories for varying values of \(\omega\). The results show that introducing a larger \(\omega\) can lead to more stable optimization, while this effect diminishes as \(\omega\) decreases. In addition, larger values of \(\omega\) yield faster reward improvement in the early training stages but converge to slightly lower final performance. A plausible explanation is that a larger \(\omega\) enables more effective learning from sparse rewards during early training, whereas the increased optimization bias introduced by an overly large \(\omega\) may mildly degrade asymptotic performance. Overall, these results indicate that adjusting \(\omega\) provides a practical mechanism for balancing bias and variance during training.
\begin{figure}[h]
    \centering

    \includegraphics[width=0.6\linewidth]{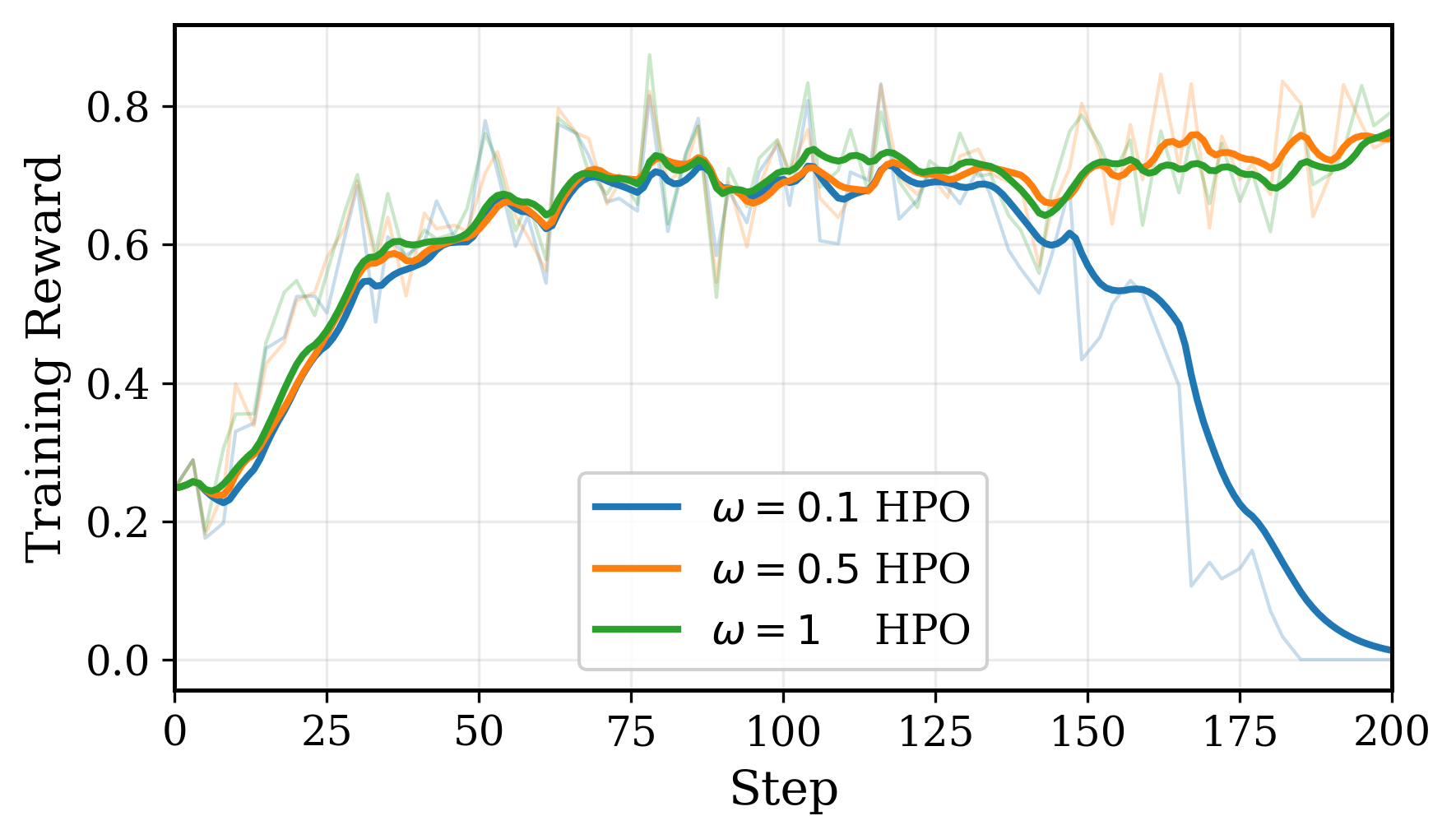}
    \caption{Training reward dynamics of Qwen2.5-3B-Instruct with different $\omega$ settings. }
    \label{fig:omega}
\end{figure}

\subsection{Case Study}
\label{app:case}
We further conduct a qualitative analysis of $A_S$ across representative samples. We find that $A_S$ exhibits strong step-level discrimination: it consistently assigns higher scores to informative and goal-progressing actions, while down-weighting failed, redundant, or uninformative steps. This behavior suggests that $A_S$ captures meaningful local learning signals and also explains why optimizing with $A_S$ alone can still drive effective policy updates.
\begin{tcolorbox}[title=Example 1,
colback=gray!5!white, colframe=gray!75!black, breakable]
\textbf{User:} In what year was the earlier of these two documentaries produced: The Life and Times of Hank Greenberg or Burden of Dreams?

\paragraph{Assistant:} $<\text{think}>$ The question is asking about the production year of two documentaries, and we need to find out which one is earlier and then provide the year of that documentary.$</\text{think}>$ $<\text{search}>$ which documentary was produced earlier, The Life and Times of Hank Greenberg or Burden of Dreams $</\text{search}>$ \\
($\mathbf{A_s}=\textcolor[rgb]{0,0.5,0}{\mathbf{+0.424}}$)

\paragraph{User:} $<\text{information}>$ \\
\textsc{Doc 1} ... The Life and Times of Hank Greenberg is a 1998 documentary film written, directed, and produced by Aviva Kempner about Hank Greenberg, first baseman of the Detroit Tigers, who was inducted into the Baseball Hall of Fame. ... \\
\textsc{Doc 2} ... \\
\textsc{Doc 3} ... \\
$</\text{information}>$

\paragraph{Assistant:} $<\text{think}>$ From the information provided, it is clear that The Life and Times of Hank Greenberg is a 1998 documentary. Now let's find out about the production year of Burden of Dreams.$</\text{think}>$ $<\text{search}>$ when was Burden of Dreams produced $</\text{search}>$ \\
($\mathbf{A_s}=\textcolor[rgb]{0,0.5,0}{\mathbf{+0.218}}$)

\paragraph{User:} $<\text{information}>$ \\
\textsc{Doc 1} ... Burden of Dreams is a 1982 “making-of” documentary film directed by Les Blank, shot during and about the chaotic production of Werner Herzog’s film Fitzcarraldo. ...\\
\textsc{Doc 2} ... \\
\textsc{Doc 3} ... \\
$</\text{information}>$

\paragraph{Assistant:}  From the information provided, it is clear that Burden of Dreams was produced in 1982. Comparing the production years of the two documentaries, The Life and Times of Hank Greenberg (1998) and Burden of Dreams (1982), it is clear that Burden of Dreams was produced earlier. Therefore, the answer is 1982. \\
($\mathbf{A_s}=\textcolor{red}{\mathbf{-0.472}}$)

\paragraph{User:} Your previous action is invalid. If you want to search, you should put the query between $<\text{search}>$ and $</\text{search}>$. If you want to give the final answer, you should put the answer between $<\text{answer}>$ and $</\text{answer}>$. Please try again.

\paragraph{Assistant:} $<\text{think}>$ From the information provided, it is clear that Burden of Dreams was produced in 1982. Comparing the production years of the two documentaries, The Life and Times of Hank Greenberg (1998) and Burden of Dreams (1982), it is clear that Burden of Dreams was produced earlier. Therefore, the answer is 1982.$</\text{think}>$ $<\text{answer}>$ 1982 $</\text{answer}>$ \\
($\mathbf{A_s}=\textcolor[rgb]{0,0.5,0}{\mathbf{+0.377}}$)

\paragraph{User:} Congratulations! You have answered the question correctly!!!

\end{tcolorbox}

\begin{tcolorbox}[title=Example 2,
colback=gray!5!white, colframe=gray!75!black, breakable]
\textbf{User:} Bone carving was used in the casket now on display in what institution?

\paragraph{Assistant:} $<\text{think}>$ I need to find out which institution has the casket with bone carving that is now on display.$</\text{think}>$ $<\text{search}>$ Bone carving was used in the casket now on display in what institution? $</\text{search}>$ \\
($\mathbf{A_s}=\textcolor[rgb]{0,0.5,0}{\mathbf{+1.20}}$)

\paragraph{User:} $<\text{information}>$ \\
\textsc{Doc 1} ... ivory casket … now in the Walters Art Museum …  \\
\textsc{Doc 2} ... a small Anglo-Saxon whale’s bone … chest … densely decorated with knife-cut … scenes … now in the \textbf{British Museum}. … \\
\textsc{Doc 3} ... \\
$</\text{information}>$

\paragraph{Assistant:} $<\text{think}>$ From the search results, it is clear that the casket with bone carving now on display is in the Walters Art Museum, Baltimore, Maryland. $</\text{think}>$ $<\text{answer}>$ Walters Art Museum $</\text{answer}>$ \\
($\mathbf{A_s}=\textcolor{red}{\mathbf{-0.574}}$)

\paragraph{User:} Sorry, your answer is incorrect.
\end{tcolorbox}

\clearpage

\begin{tcolorbox}[title=Example 1: the advantage of GRPO,
colback=gray!5!white, colframe=gray!75!black, breakable]
\textbf{User:} In what year was the earlier of these two documentaries produced: The Life and Times of Hank Greenberg or Burden of Dreams?

\paragraph{Assistant:} $<\text{think}>$ The question is asking about the production year of two documentaries, and we need to find out which one is earlier and then provide the year of that documentary.$</\text{think}>$ $<\text{search}>$ which documentary was produced earlier, The Life and Times of Hank Greenberg or Burden of Dreams $</\text{search}>$ \\
($\mathbf{A_s}=\textcolor[rgb]{0,0.5,0}{\mathbf{+1.29}}$)

\paragraph{User:} $<\text{information}>$ \\
\textsc{Doc 1} ... The Life and Times of Hank Greenberg is a 1998 documentary film written, directed, and produced by Aviva Kempner about Hank Greenberg, first baseman of the Detroit Tigers, who was inducted into the Baseball Hall of Fame. ... \\
\textsc{Doc 2} ... \\
\textsc{Doc 3} ... \\
$</\text{information}>$

\paragraph{Assistant:} $<\text{think}>$ From the information provided, it is clear that The Life and Times of Hank Greenberg is a 1998 documentary. Now let's find out about the production year of Burden of Dreams.$</\text{think}>$ $<\text{search}>$ when was Burden of Dreams produced $</\text{search}>$ \\
($\mathbf{A_s}=\textcolor[rgb]{0,0.5,0}{\mathbf{+1.29}}$)

\paragraph{User:} $<\text{information}>$ \\
\textsc{Doc 1} ... Burden of Dreams is a 1982 “making-of” documentary film directed by Les Blank, shot during and about the chaotic production of Werner Herzog’s film Fitzcarraldo. ...\\
\textsc{Doc 2} ... \\
\textsc{Doc 3} ... \\
$</\text{information}>$

\paragraph{Assistant:}  From the information provided, it is clear that Burden of Dreams was produced in 1982. Comparing the production years of the two documentaries, The Life and Times of Hank Greenberg (1998) and Burden of Dreams (1982), it is clear that Burden of Dreams was produced earlier. Therefore, the answer is 1982. \\
($\mathbf{A_s}=\textcolor[rgb]{0,0.5,0}{\mathbf{+1.29}}$)

\paragraph{User:} Your previous action is invalid. If you want to search, you should put the query between $<\text{search}>$ and $</\text{search}>$. If you want to give the final answer, you should put the answer between $<\text{answer}>$ and $</\text{answer}>$. Please try again.

\paragraph{Assistant:} $<\text{think}>$ From the information provided, it is clear that Burden of Dreams was produced in 1982. Comparing the production years of the two documentaries, The Life and Times of Hank Greenberg (1998) and Burden of Dreams (1982), it is clear that Burden of Dreams was produced earlier. Therefore, the answer is 1982.$</\text{think}>$ $<\text{answer}>$ 1982 $</\text{answer}>$ \\
($\mathbf{A_s}=\textcolor[rgb]{0,0.5,0}{\mathbf{+1.29}}$)

\paragraph{User:} Congratulations! You have answered the question correctly!!!

\end{tcolorbox}

\clearpage

\begin{tcolorbox}[title=Example 2: the advantage of PPO,
colback=gray!5!white, colframe=gray!75!black, breakable]
\textbf{User:} In what year was the earlier of these two documentaries produced: The Life and Times of Hank Greenberg or Burden of Dreams?

\paragraph{Assistant:} $<\text{think}>$ The question is asking about the production year of two documentaries, and we need to find out which one is earlier and then provide the year of that documentary.$</\text{think}>$ $<\text{search}>$ which documentary was produced earlier, The Life and Times of Hank Greenberg or Burden of Dreams $</\text{search}>$ \\
($\mathbf{A_s}=\textcolor[rgb]{0,0.5,0}{\mathbf{+0.821}}$)

\paragraph{User:} $<\text{information}>$ \\
\textsc{Doc 1} ... The Life and Times of Hank Greenberg is a 1998 documentary film written, directed, and produced by Aviva Kempner about Hank Greenberg, first baseman of the Detroit Tigers, who was inducted into the Baseball Hall of Fame. ... \\
\textsc{Doc 2} ... \\
\textsc{Doc 3} ... \\
$</\text{information}>$

\paragraph{Assistant:} $<\text{think}>$ From the information provided, it is clear that The Life and Times of Hank Greenberg is a 1998 documentary. Now let's find out about the production year of Burden of Dreams.$</\text{think}>$ $<\text{search}>$ when was Burden of Dreams produced $</\text{search}>$ \\
($\mathbf{A_s}=\textcolor[rgb]{0,0.5,0}{\mathbf{+0.762}}$)

\paragraph{User:} $<\text{information}>$ \\
\textsc{Doc 1} ... Burden of Dreams is a 1982 “making-of” documentary film directed by Les Blank, shot during and about the chaotic production of Werner Herzog’s film Fitzcarraldo. ...\\
\textsc{Doc 2} ... \\
\textsc{Doc 3} ... \\
$</\text{information}>$

\paragraph{Assistant:}  From the information provided, it is clear that Burden of Dreams was produced in 1982. Comparing the production years of the two documentaries, The Life and Times of Hank Greenberg (1998) and Burden of Dreams (1982), it is clear that Burden of Dreams was produced earlier. Therefore, the answer is 1982. \\
($\mathbf{A_s}=\textcolor[rgb]{0,0.5,0}{\mathbf{+0.684}}$)

\paragraph{User:} Your previous action is invalid. If you want to search, you should put the query between $<\text{search}>$ and $</\text{search}>$. If you want to give the final answer, you should put the answer between $<\text{answer}>$ and $</\text{answer}>$. Please try again.

\paragraph{Assistant:} $<\text{think}>$ From the information provided, it is clear that Burden of Dreams was produced in 1982. Comparing the production years of the two documentaries, The Life and Times of Hank Greenberg (1998) and Burden of Dreams (1982), it is clear that Burden of Dreams was produced earlier. Therefore, the answer is 1982.$</\text{think}>$ $<\text{answer}>$ 1982 $</\text{answer}>$ \\
($\mathbf{A_s}=\textcolor[rgb]{0,0.5,0}{\mathbf{+0.707}}$)

\paragraph{User:} Congratulations! You have answered the question correctly!!!

\end{tcolorbox}

\clearpage

\begin{tcolorbox}[title=Example 3: the advantage of HPO,
colback=gray!5!white, colframe=gray!75!black, breakable]
\textbf{User:} In what year was the earlier of these two documentaries produced: The Life and Times of Hank Greenberg or Burden of Dreams?

\paragraph{Assistant:} $<\text{think}>$ The question is asking about the production year of two documentaries, and we need to find out which one is earlier and then provide the year of that documentary.$</\text{think}>$ $<\text{search}>$ which documentary was produced earlier, The Life and Times of Hank Greenberg or Burden of Dreams $</\text{search}>$ \\
($\mathbf{A_s}=\textcolor[rgb]{0,0.5,0}{\mathbf{+0.424}}$)

\paragraph{User:} $<\text{information}>$ \\
\textsc{Doc 1} ... The Life and Times of Hank Greenberg is a 1998 documentary film written, directed, and produced by Aviva Kempner about Hank Greenberg, first baseman of the Detroit Tigers, who was inducted into the Baseball Hall of Fame. ... \\
\textsc{Doc 2} ... \\
\textsc{Doc 3} ... \\
$</\text{information}>$

\paragraph{Assistant:} $<\text{think}>$ From the information provided, it is clear that The Life and Times of Hank Greenberg is a 1998 documentary. Now let's find out about the production year of Burden of Dreams.$</\text{think}>$ $<\text{search}>$ when was Burden of Dreams produced $</\text{search}>$ \\
($\mathbf{A_s}=\textcolor[rgb]{0,0.5,0}{\mathbf{+0.218}}$)

\paragraph{User:} $<\text{information}>$ \\
\textsc{Doc 1} ... Burden of Dreams is a 1982 “making-of” documentary film directed by Les Blank, shot during and about the chaotic production of Werner Herzog’s film Fitzcarraldo. ...\\
\textsc{Doc 2} ... \\
\textsc{Doc 3} ... \\
$</\text{information}>$

\paragraph{Assistant:}  From the information provided, it is clear that Burden of Dreams was produced in 1982. Comparing the production years of the two documentaries, The Life and Times of Hank Greenberg (1998) and Burden of Dreams (1982), it is clear that Burden of Dreams was produced earlier. Therefore, the answer is 1982. \\
($\mathbf{A_s}=\textcolor{red}{\mathbf{-0.472}}$)

\paragraph{User:} Your previous action is invalid. If you want to search, you should put the query between $<\text{search}>$ and $</\text{search}>$. If you want to give the final answer, you should put the answer between $<\text{answer}>$ and $</\text{answer}>$. Please try again.

\paragraph{Assistant:} $<\text{think}>$ From the information provided, it is clear that Burden of Dreams was produced in 1982. Comparing the production years of the two documentaries, The Life and Times of Hank Greenberg (1998) and Burden of Dreams (1982), it is clear that Burden of Dreams was produced earlier. Therefore, the answer is 1982.$</\text{think}>$ $<\text{answer}>$ 1982 $</\text{answer}>$ \\
($\mathbf{A_s}=\textcolor[rgb]{0,0.5,0}{\mathbf{+0.377}}$)

\paragraph{User:} Congratulations! You have answered the question correctly!!!

\end{tcolorbox}

\end{document}